%% file: icml2022.tex
\documentclass[nohyperref]{article}

\usepackage{microtype}
\usepackage{graphicx}
\usepackage{subfigure}
\usepackage{booktabs} % for professional tables

\usepackage{hyperref}

\usepackage[accepted]{icml2022}

\usepackage{amsmath}
\usepackage{amssymb}
\usepackage{mathtools}
\usepackage{amsthm}
\usepackage{amsfonts}

\usepackage[capitalize,noabbrev]{cleveref}

\theoremstyle{plain}
\newtheorem{theorem}{Theorem}[section]

\newtheorem{lemma}[theorem]{Lemma}
\newtheorem{remark}[theorem]{Remark}

\theoremstyle{definition}
\newtheorem{definition}[theorem]{Definition}

\usepackage[textsize=tiny]{todonotes}

\usepackage{smile}
\usepackage{extarrows}
\usepackage[OT1]{fontenc}
\usepackage{algorithm}
\usepackage{algorithmic}
\usepackage{scalerel}
\RequirePackage{times}
\RequirePackage{natbib}
 \usepackage{hyperref}
\hypersetup{
  pdftex,
  pdffitwindow=true,
  pdfstartview={FitH},
  pdfnewwindow=true,
  colorlinks,
  linktocpage=true,
  linkcolor=blue,
  urlcolor=blue,
  citecolor=blue
}
\usepackage[capitalize]{cleveref} 
\RequirePackage{times}
\RequirePackage{natbib}
\usepackage[textsize=tiny]{todonotes}

\newcommand{\mI}{\mathbb{I}}

\newcommand{\mH}{\mathbb{H}}

\newcommand{\mix}{\text{mix}}

\newcommand{\KL}{\mathrm{KL}}

\newcommand{\supp}{\text{supp}}

\newcommand{\TS}{\text{TS}}

\DeclareMathOperator{\poly}{poly}

\newcommand{\BR}{\mathfrak{BR}}

\def\purple{\color{purple}}

\icmltitlerunning{Contextual Information-Directed Sampling}

\begin{document}

\twocolumn[
\icmltitle{Contextual Information-Directed Sampling}



\icmlsetsymbol{equal}{*}

\begin{icmlauthorlist}
\icmlauthor{Botao Hao}{yyy}
\icmlauthor{Tor Lattimore}{yyy}
\icmlauthor{Chao Qin}{comp}
\end{icmlauthorlist}

\icmlaffiliation{yyy}{Deepmind}
\icmlaffiliation{comp}{Columbia University}

\icmlcorrespondingauthor{Botao Hao}{haobotao000@gmail.com}

\icmlkeywords{Machine Learning, ICML}

\vskip 0.3in
]



\printAffiliationsAndNotice{\icmlEqualContribution} % otherwise use the standard text.

\begin{abstract}
Information-directed sampling (IDS) has recently demonstrated its potential as a data-efficient reinforcement learning algorithm \citep{lu2021reinforcement}. However, it is still unclear what is the right form of information ratio to optimize when contextual information is available. We investigate the IDS design through two contextual bandit problems: contextual bandits with graph feedback and sparse linear contextual bandits. We provably demonstrate the advantage of \emph{contextual IDS} over \emph{conditional IDS} and emphasize the importance of considering the context distribution. The main message is that an intelligent agent should invest more on the actions that are beneficial for the future unseen contexts while the conditional IDS can be myopic. We further propose a computationally-efficient version of contextual IDS based on Actor-Critic and evaluate it empirically on a neural network contextual bandit. 
\end{abstract}

\input{intro}
\input{setting}

\input{contextual_IDS}

\input{graph_feedback}

\input{sparse_bandits}

\input{algorithm}

\input{experiment}
\input{discussion}

\bibliography{ref}
\bibliographystyle{icml2022}
%%%%%%%%%%%%%%%%%%%%%%%%%%%%%%%%%%%%%%%%%%%%%%%%%%%%%%%%%%%%%%%%%%%%%%%%%%%%%%%
%%%%%%%%%%%%%%%%%%%%%%%%%%%%%%%%%%%%%%%%%%%%%%%%%%%%%%%%%%%%%%%%%%%%%%%%%%%%%%%
% APPENDIX
%%%%%%%%%%%%%%%%%%%%%%%%%%%%%%%%%%%%%%%%%%%%%%%%%%%%%%%%%%%%%%%%%%%%%%%%%%%%%%%
%%%%%%%%%%%%%%%%%%%%%%%%%%%%%%%%%%%%%%%%%%%%%%%%%%%%%%%%%%%%%%%%%%%%%%%%%%%%%%%
\newpage
\appendix
\onecolumn
\input{appendix}

%\input{chapter/appendix2}
%%%%%%%%%%%%%%%%%%%%%%%%%%%%%%%%%%%%%%%%%%%%%%%%%%%%%%%%%%%%%%%%%%%%%%%%%%%%%%%
%%%%%%%%%%%%%%%%%%%%%%%%%%%%%%%%%%%%%%%%%%%%%%%%%%%%%%%%%%%%%%%%%%%%%%%%%%%%%%%

\end{document}

%% file: intro.tex
\section{Introduction}
Information-directed sampling (IDS) \citep{russo2018learning} is a promising approach to balance the exploration and exploitation tradeoff  \citep{lu2021reinforcement}. Its theoretical properties have been systematically studied in a range of problems \citep{kirschner2018information, liu2018information, kirschner2020information, kirschner2020asymptotically, hao2021information}. However, the aforementioned analysis\footnote{One exception is \citet{kirschner2020information} who has extended IDS to contextual partial monitoring.} is limited to the fixed action set.

In this work, we study the IDS design for contextual bandits \citep{langford2007epoch}, which can be viewed as a simplified case of full reinforcement learning. The context at each round is generated independently, rather than determined by the actions. Contextual bandits are widely used in real-world
applications, such as recommender systems \citep{li2010contextual}.

A natural generalization of IDS to the contextual case is \emph{conditional IDS}. Given the current context, conditional IDS acts as if it were facing a fixed action set.
However, this design ignores the context distribution and can be myopic. We mainly want to address the question of what is the right form of information ratio in the contextual setting and highlight the necessity of utilizing the context distribution. 
\paragraph{Contributions} Our contribution is four-fold:
\begin{itemize}
    \item We introduce a version of contextual IDS that takes the context distribution into consideration 
    for contextual bandits with graph feedback and sparse linear contextual bandits.
    \item For contextual bandits with graph feedback, we prove that conditional IDS suffers $\Omega(\sqrt{\beta(\cG)n})$ Bayesian regret lower bound for a particular prior and graph while contextual IDS can achieve $\tilde{O}(\min\{\sqrt{\beta(\cG)n},\delta(\cG)^{1/3}n^{2/3})\}$ Bayesian regret upper bound for any prior. 
    Here, $n$ is the time horizon, $\cG$ is a directed feedback
graph over the set of actions, $\beta(\cG)$ is the independence number and $\delta(\cG)$ is the weak domination number of the graph. In the regime where $\beta(\cG)\gtrsim (\delta(\cG)^2n)^{1/3}$, contextual IDS achieves better regret bound than conditional IDS.
    \item For sparse linear contextual bandits, we prove that conditional IDS suffers $\Omega(\sqrt{nds})$ Bayesian regret lower bound for a particular sparse prior while contextual IDS can achieve $\tilde{O}(\min\{\sqrt{nds}, sn^{2/3}\})$ Bayesian regret upper bound for any sparse prior. %\todoc{When $n$ is large, this upper bound is not better. I guess we need to emphasize the regime in which it is better?}
    Here, $d$ is the feature dimension and $s$ is the sparsity. In the data-poor regime where $d\gtrsim sn^{1/3}$, contextual IDS achieves better regret bound than conditional IDS.
    \item We further propose a computationally-efficient algorithm to approximate contextual IDS based on Actor-Critic \citep{konda2000actor} and evaluate it empirically on a neural network contextual bandit.
\end{itemize}

\section{Related works}
\citet{russo2018learning} introduced IDS and derived Bayesian regret bounds for multi-armed bandits, linear bandits and combinatorial bandits.  \citet{kirschner2018information, kirschner2020information} investigated the use of frequentist IDS for bandits with heteroscedastic noise and partial monitoring. \citet{kirschner2020asymptotically} proved the asymptotic optimality of frequentist IDS for linear bandits. \citet{hao2021information} developed a class of information-theoretic Bayesian regret bounds for sparse linear bandits that nearly match existing lower bounds on a variety of problem instances. Recently, \citet{lu2021reinforcement} proposed a general reinforcement learning framework and used IDS as a key component to build a data-efficient agent. \citet{arumugam2021value} investigated different versions of learning targets when designing IDS. However, there are no concrete regret bounds provided in those works.

\paragraph{Graph feedback}
\citet{mannor2011bandits, alon2015online} gave a full characterization of online learning with graph feedback. \citet{tossou2017thompson} provided the first information-theoretic analysis of Thompson sampling for bandits with graph feedback and \citet{liu2018information} derived a Bayesian regret bound of IDS. Both bounds depend on the clique cover number which could be much larger than the independence number. \citet{liu2018analysis} further improved the analysis with the feedback graph unknown and derived the Bayesian regret bound of a mixture policy in terms of the independence number. In contrast, our work is the first one to consider contextual bandits with graph feedback and derived nearly optimal regret bound in terms of both independence number and domination number for a single algorithm. Very recently, \citet{dann2020reinforcement} considered the reinforcement learning with graph feedback. Their $O(\sqrt{n})$-upper bound depends on mas-number rather than independence number. We briefly summarize the comparison in Table \ref{table:comparsion2}.

\begin{table*}[h]\label{table:comparsion2}
%\centering
\caption{Comparisons with existing results on regret upper bounds and lower bounds for bandits with graphical feedback. Here, $\beta(\cG)$ is the independence number, $\delta(\cG)$ is the weak domination number and $c(\cG)$ is the clique cover number of the graph. Note that $\beta(\cG)\leq c(\cG)$.}
\vspace{0.1in}
\scalebox{0.9}{
\begin{tabular}{ |l|c|c|c| } 
 \hline
 \textbf{Upper Bound}& Regret& Contextual?&  Algorithm\\ 
 \hline
 \citet{alon2015online} & $\tilde O(\sqrt{\beta(\cG)n})$ & no & Exp3.G for strongly observable graph\\ 
 \hline
 \citet{alon2015online} & $\tilde O(\delta(\cG)^{1/3}n^{2/3})$ & no & Exp3.G for weakly observable graph\\ 
 \hline
 \citet{tossou2017thompson} & $\tilde O(\sqrt{c(\cG)n})$ & no & Thompson sampling\\ 
 \hline
 \citet{liu2018information} & $\tilde O(\sqrt{c(\cG)n})$ & no & fixed action-set IDS\\ 
 \hline
 {\purple This paper} & {\purple$ \tilde O(\min(\sqrt{\beta(\cG)n}, \delta(\cG)^{1/3}n^{2/3}))$} & {\purple yes} & {\purple contextual IDS}\\ 
  \hline

 \textbf{Minimax Lower Bound}&& &  \\ 
 \hline
 \citet{alon2015online} & $\Omega(\sqrt{\beta(\cG)n})$ & no & strongly observable graph \\ 
 \hline
 \citet{alon2015online} & $\Omega(\delta(\cG)^{1/3}n^{2/3})$ & no & weakly observable graph \\ 
 \hline
\end{tabular}}
\end{table*}

\paragraph{Sparse linear bandits}
For sparse linear bandits with fixed action set, \citet{abbasi2012online} proposed an inefficient
online-to-confidence-set conversion approach that achieves an $\tilde{O}(\sqrt{sdn})$ upper bound for an arbitrary action set. \citet{lattimore2015linear}
developed a selective explore-then-commit algorithm that only works when the action set is exactly the binary hypercube and derived an optimal $O(s\sqrt{n})$ upper bound. \citet{hao2020high} introduced the notion of an exploratory action set and proved a $\Theta(\poly(s)n^{2/3})$ minimax rate for the data-poor regime using an explore-then-commit algorithm. Note that the minimax lower bound automatically holds for contextual setting when the context distribution is a Dirac on the hard problem instance. \citet{hao2021online} extended this concept to a MDP setting. 

It recently
became popular to study the contextual setting. These results can not be reduced to our setting since they rely on either careful assumptions on the context distribution to achieve $\tilde{O}(\poly(s)\sqrt{n})$ or $\tilde{O}(\poly(s)\log(n))$ regret bounds \citep{bastani2020online, wang2018minimax, kim2019doubly, wang2020nearly, ren2020dynamic, oh2021sparsity} such that classical high-dimensional statistics can be used. However, minimax lower bounds is missing to understand if those assumptions are fundamental. Another line of works have polynomial dependency on the number of actions \citep{agarwal2014taming, foster2020beyond, simchi2020bypassing}. We briefly summarize the comparison in Table \ref{table:comparsion}.

\begin{table*}[h]\label{table:comparsion}
\centering
\caption{Comparisons with existing results on regret upper bounds and lower bounds for sparse linear bandits. Here, $s$ is the sparsity, $d$ is the feature dimension, $n$ is the number of rounds, $K$ is the number of arms, $C_{\min}$ is the minimum eigenvalue of the data matrix and $\tau$ is a problem-dependent parameter that may have a complicated form.}
\vspace{0.1in}
\scalebox{0.9}{
\begin{tabular}{ |l|c|c|c|c| } 
 \hline
 \textbf{Upper Bound}& Regret& Contextual?& Assumption & Algorithm\\ 
 \hline
 \citet{abbasi2012online} & $\tilde O(\sqrt{sdn})$ & yes & none &  UCB\\ 
 \hline
 \citet{sivakumar2020structured} & $\tilde O(\sqrt{sdn})$ & yes & adver. + Gaussian noise&  greedy \\ 
  \hline
 \citet{bastani2020online} & $\tilde O(\tau K s^2(\log(n))^2)$ & yes & compatibility condition&  Lasso greedy \\
 \hline
 \citet{lattimore2015linear}
  &$\tilde O(s\sqrt{n})$& no& action set is hypercube & explore-then-commit\\
  \hline
   \citet{hao2021information}
  &$\tilde O(\min(\sqrt{sdn}, sn^{2/3}/C_{\min}))$& no & none  &  fixed action-set IDS\\
  \hline
  {\purple This paper}
  &{\purple $\tilde O(\min(\sqrt{sdn}, sn^{2/3}/C_{\min}))$}&{\purple yes} & {\purple none} & {\purple contextual IDS}\\
  \hline
 \textbf{Minimax Lower Bound}&& & &  \\ 
 \hline
 \citet{hao2020high} & $\Omega(\min(\sqrt{sdn},C_{\min}^{-1/3}s^{1/3}n^{2/3}))$ & no &  N.A & N.A. \\ 
 \hline
\end{tabular}}
\end{table*}

%% file: setting.tex
\section{Problem Setting}
We study the stochastic contextual bandit problem with a horizon of $n$ rounds and a finite set of possible contexts $\cS$. For each context $m\in \cS$, there is an action set $\cA^m$. The interaction protocol is as follows. First the environment samples a sequence of independent contexts $(s_t)_{t=1}^n$ from a distribution $\xi$ over $\cS$.
At the start of round $t$, the context $s_t$ is revealed to the agent, who may use their observations and possibly an external source of randomness to choose an action $A_t \in \cA_t=\cA^{s_t}$. Then the agent receives an observation $O_{t,A_t}$ including an immediate reward $Y_{t, A_t}$ as well as some side information. Here
\begin{equation*}
Y_{t, a} = f(s_t, a, \theta^*) + \eta_{t,a}\,, 
\end{equation*}
where $f$ is the reward function, $\theta^*$ is the unknown parameter and $\eta_{t,a}$ is 1-sub-Gaussian noise. Sometimes we write $Y_t=Y_{t, A_t}$ for short.

We consider the Bayesian setting in the sense that $\theta^*$ is sampled from some prior distribution. Let $\cA=\cup_{m\in\cS} \cA^m$ and the history $\cF_t = \{(s_1, A_1, O_1), \ldots, (s_{t-1}, A_{t-1}, O_{t-1})\}$.
A policy $\pi= (\pi_t)_{t\in\mathbb N}$ is a sequence of (suitably measurable) deterministic mappings from the history $\cF_t$ and context space $\cS$ to a probability distribution $\cP(\cA)$ with the constraint that for each context $m$, $\pi_t(\cdot|m)$ can only put mass over $\cP(\cA^m)$.
Then the \emph{Bayesian regret} of a policy $\pi$ is defined as
\begin{equation}\label{eqn:bayesian_regret}
    \BR(n; \pi) = \mathbb E\left[\sum_{t=1}^n \max_{a\in\cA_t}f(s_t, a, \theta^*)-\sum_{t=1}^nY_t \right]\,,
\end{equation}
where the expectation is over the interaction sequence induced by the agent and environment, the context distribution and the prior distribution over $\theta^*$.

\paragraph{Notations} Given a measure $\mathbb P$ and jointly distributed random variables $X$ and $Y$ we let $\mathbb P_X$ denote the law of $X$ and we let $\mathbb P_{X|Y}$ be the conditional law of $X$ given $Y$ such that: $\mathbb P_{X|Y}(\cdot) = \mathbb P(X\in \cdot|Y)$. The mutual information between $X$ and $Y$ is $\mI(X;Y) = \mathbb E[D_{\KL}(\mathbb P_{X|Y}||\mathbb P_X)]$ where $D_{\KL}$ is the relative entropy. 
We write $\mathbb P_t(\cdot) = \mathbb P(\cdot|\cF_t)$ as the posterior measure where $\mathbb P$ is the probability measure over $\theta^*$ and the history and $\mathbb E_t[\cdot] = \mathbb E[\cdot|\cF_t]$. Denote $\mI_t(X;Y) = \mathbb E_t[D_{\KL}(\mathbb P_{t, X|Y}||\mathbb P_{t,X})]$. We write $\ind\{\cdot\}$ as an indicator function.  For a positive semidefinite matrix $A$, we let $\sigma_{\min}(A)$ be the minimum eigenvalue of $A$.  Denote $\cP(\cA)$ be the space of probability measures over a set $\cA$ with the Borel $\sigma$-algebra.

%% file: contextual_IDS.tex
\section{Conditional IDS versus Contextual IDS}
We introduce the formal definition of conditional IDS and contextual IDS for general contextual bandits. Suppose that the optimal action associated with context $s_t$ is $a_t^* =\argmax_{a\in\cA_t} f(s_t, a, \theta^*)$, which in Bayesian setting is a random variable. We first define the one-step expected regret of taking action $a\in\cA_t$ as 
\begin{equation}\label{def:one_step_regret}
\begin{split}
      \Delta_t(a|s_t) :=\mathbb E_t\left[f(s_t, a_t^*, \theta^*)-f(s_t, a, \theta^*)\Big| s_t\right]\,,
\end{split}
\end{equation}
and write $\Delta_t(s_t)\in\mathbb R^{|\cA_t|}$ as the corresponding vector.
\subsection{Conditional IDS}\label{sec:conditional_IDS}
A natural generalization of the standard IDS design \citep{russo2018learning} to contextual setting is \emph{conditional IDS}. It has been investigated empirically in \citet{lu2021reinforcement} for the full reinforcement learning setting.  

Conditional on the current context, we define the information gain of taking action $a$ with respect to the current optimal action $a_t^*$ as $\mI_t(a_t^*;O_{t,a}|s_t):=$
\begin{equation*}
\begin{split}
\mathbb E_{t}\left[D_{\KL}(\mathbb P_t(a_t^*\in\cdot|s_t, O_{t,a})||\mathbb P_t(a_t^*\in\cdot|s_t))\Big|s_t\right]\,,
\end{split}
\end{equation*}
and write $\mI_t(a_t^*,s_t)\in\mathbb R^{|\cA_t|}$ such that $[\mI_t(a_t^*, s_t)]_a = \mI_t(a_t^*;O_{t,a}|s_t)$.

We introduce the \emph{$(\alpha, \lambda)$-conditional information ratio} (CIR) as $\Gamma_{t,\alpha}^{\lambda}: \cP(\cA)\to\mathbb R$ such that
\begin{equation}\label{def:conditional_IDS}
    \Gamma_{t,\alpha}^{\lambda}(\pi(\cdot|s_t)) = \frac{ \max\left(0, \Delta_t(s_t)^{\top}\pi(\cdot|s_t)-\alpha\right)^{\lambda}}{\mI_t(a_t^*, s_t)^{\top}\pi(\cdot|s_t) }\,,
\end{equation}
for some parameters $\alpha, \lambda>0$. Conditional IDS minimizes $(\alpha, \lambda)$-CIR to find a probability distribution over the action space:
\begin{equation*}
    \pi_{t}(\cdot|s_t) = \argmin_{\pi(\cdot|s_t)\in\cP(\cA_t)}\Gamma_{t,\alpha}^{\lambda}(\pi(\cdot|s_t))\,.
\end{equation*}
\begin{remark}
In comparison to the standard information ratio by \citet{russo2018learning} who specified $\lambda=2, \alpha=0$ in the non-contextual setting, we consider a generalized version introduced by \citet{lattimore2020mirror}. As observed by \citet{lattimore2020mirror}, the right value of $\lambda$ depends on the dependence of the regret on the horizon. The parameter $\alpha$ is always chosen at order $O(1/\sqrt{n})$ for certain problems such that its regret is negligible. Their role will be more clear in Sections \ref{sec:graph_feedback}\&\ref{sec:sparse_bandits}.  %As observed by \cite{lattimore2020mirror}, the right value of $\lambda$ depends on the dependence of the regret on the horizon.
\end{remark}
\subsection{Contextual IDS}\label{sec:contextual_IDS}
A possible limitation of conditional IDS is that the CIR does not take the context distribution into consideration. As we show later, this can result in sub-optimality in certain problems. \emph{Contextual IDS}, which was first proposed by \citet{kirschner2020information} for partial monitoring, is introduced to remedy this shortcoming.

Let $\pi^*:\cS\to\cA$ be the optimal policy such that for each $m\in\cS$, $\pi^*(m)=\argmax_{a\in\cA^m} f(m, a, \theta^*)$ with the tie broken arbitrarily.
We define the information gain of taking action $a$ with respect to $\pi^*$ as:
\begin{equation*}
\begin{split}
 \mI_t(\pi^*;O_{t,a}):=\mathbb E_{t}\left[D_{\KL}(\mathbb P_{t}(\pi^*\in\cdot|O_{t,a})||\mathbb P_{t}(\pi^*\in\cdot))\right]\,,
\end{split}
\end{equation*}
and write $\mI_t(\pi^*)\in\mathbb R^{|\cA_t|}$ as the corresponding vector. Define a policy class $\Pi=\{\pi:\cS\to\cP(\cA), \supp(\pi(\cdot|m))\subset \cA^m\}$ where $\supp(\cdot)$ is the support of a vector.
%where $\mathbb P_{t, a^*|O_{t,a}}=\mathbb P(a^*\in\cdot|\cF_t, O_{t,a})$.
We introduce the \emph{$(\alpha, \lambda)$-marginal information ratio} (MIR) as  $\Psi_{t,\alpha}^{\lambda}: \Pi\to\mathbb R$ such that
\begin{equation}\label{def:MIR}
    \Psi_{t,\alpha}^{\lambda}(\pi) =  \frac{ \max\left(0, (\mathbb E_{s_t}[\Delta_t(s_t)^{\top}\pi(\cdot|s_t)]-\alpha)\right)^{\lambda}}{\mathbb E_{s_t}[\mI_t(\pi^*)^{\top}\pi(\cdot|s_t)]}\,,
\end{equation}
where the expectation is taken with respect to the context distribution.
Contextual IDS minimizes the $(\alpha, \lambda)$-MIR to find a mapping from the context space to the action space:
\begin{equation*}
    \pi_t = \argmin_{\pi\in\Pi}\Psi_{t,\alpha}^{\lambda}(\pi)\,.
\end{equation*}
When receiving context $s_t$ at round $t$, the agent plays actions according to $\pi_t(\cdot|s_t)$. 

We highlight the key difference between conditional IDS and contextual IDS as follows:
\begin{itemize}
    \item Conditional IDS only optimizes a probability distribution for the current context while contextual IDS optimizes a full mapping from the context space to action space.
    \item Conditional IDS only seeks information about the current optimal action while contextual IDS seeks information for the whole optimal policy.
\end{itemize}
\begin{remark}
We can also define the information gain in conditional IDS and contextual IDS with respect to the unknown parameter $\theta^*$. This could bring in certain computational advantage when approximating the information gain. However, $\theta^*$ contains much more information to learn than the optimal action or policy such that the agent may suffer more regret.
\end{remark}

\subsection{Why conditional IDS could be myopic?} \label{sec:myopic}

Conditional IDS myopically balances exploration and exploitation without taking the context distribution into consideration. This has the advantage of being simple to implement but sometimes leads both under and over exploration, which the following examples illustrate.
In both examples there are two contexts arriving with equal probability. 
\begin{itemize}
    \item \textbf{Example 1} [\textsc{under exploration}]\,\, Consider a noiseless case. Context set 1 contains $k$ actions where one is the optimal action and the remaining $k-1$ actions yield regret 1. Context set 2 contains a revealing action with regret 1 and one action with no regret. The revealing action provides an observation of the rewards for all the $k$ actions in context set 1.
    When context set 2 arrives, conditional IDS will never play the revealing action since it incurs high immediate regret with no useful information for the current context set. However, this ignores the fact that the revealing action could be informative for the unseen context set 1. Conditional IDS \emph{under-explores} and suffers $O(k)$ regret. In contrast, contextual IDS exploits the context distribution and plays the revealing action in context 2 and only suffers $O(1)$ regret.
\item \textbf{Example 2} [\textsc{over exploration}]\,\, Context set 1 contains a single revealing action (hence no regret). Context set 2 has $k$ actions. The first is a revealing action and has a (known) regret of $\Theta(\sqrt{k} \Delta)$ with $\Delta = \Theta(1/\sqrt{n})$. Of the remaining actions, one is optimal (zero regret) and the others have regret $\Delta$, with the prior such that the identify of the optimal action is unknown. Contextual IDS will avoid the revealing action in context set 2 because it understands that this information can be obtained more cheaply in context set 1. Its regret is $O(\sqrt{n})$. Meanwhile, if the constants are tuned appropriately, then conditional IDS will play the revealing action in context set 2 and suffer regret $\Omega(\sqrt{nk})$.
\end{itemize}

\begin{remark}
Similar shortcoming could also hold for the class of UCB and Thompson sampling algorithms since they have no explicit way to encode the information of context distribution.
\end{remark}

%\begin{remark}
%One may wonder if we could still define the information gain in MIR with respect to $a_t^*$ such that the denominator in Eq.~\eqref{def:MIR} is $\mathbb E[\mI_t(a_t^*, s_t)^{\top}\pi(\cdot|s_t)]$. Another choice of the learning target is the unknown parameter $\theta^*$. However, $\theta^*$ contains much more information to learn than the optimal action or policy.
%\end{remark}

\subsection{Generic regret bound for contextual IDS}
Before we step into specific examples, we first provide a generic bound for contextual IDS. Denote $\cI_{\alpha, \lambda}$ as the worst-case MIR such that for any $t\in[n]$, $\Psi_{t,\alpha}^{\lambda}(\pi_t)\leq \cI_{\alpha, \lambda}$ almost surely. 
\begin{theorem}\label{thm:generic_bound_ids}
Let $\pi^{\text{MIR}}=(\pi_t)_{t\in[n]}$ be the contextual IDS policy such that $\pi_t$ minimizes $(\alpha, 2)$-MIR at each round. 
Then the following regret upper bound holds
 \begin{equation*}
 \begin{split}
     & \BR(n; \pi^{\text{MIR}})\leq n\alpha+\\
     &\inf_{\lambda\geq 2}2^{1-2/\lambda}\cI_{\alpha, \lambda}^{1/\lambda}n^{1-1/\lambda}\mathbb E\left[\sum_{t=1}^n\mathbb E_{s_t}\left[\mI_t(\pi^*)^{\top}\pi_t(\cdot|s_t)\right]\right]^{\tfrac{1}{\lambda}}\,.
 \end{split}
 \end{equation*}
\end{theorem}
The proof is deferred to Appendix \ref{sec:proof_generic_bound}. An interesting consequence of Theorem \ref{thm:generic_bound_ids} is that contextual IDS minimizing $\lambda=2$ can adapt to $\lambda>2$. This shows the power of contextual IDS to adapt to different information-regret structures.

%% file: graph_feedback.tex
\section{Contextual Bandits with Graph Feedback}\label{sec:graph_feedback}
We consider a Bayesian formulation of contextual $k$-armed bandits with graph feedback, which generalizes the setup in \citet{alon2015online}. 
Let $\cG= (\cA, \cE)$ be a directed feedback graph over the set of actions $\cA$ with $|\cA|=k$. For each $i\in\cA$, we define its in-neighborhood and out-neighborhood as $N^{\text{in}}(i)=\{j\in\cA: (j,i)\in \cE\}$ and $N^{\text{out}}(i)=\{j\in\cA: (i,j)\in \cE\}$. The feedback graph $\cG$ is fixed and known to the agent.

Let the unknown parameter $\theta^*\in\mathbb R^{k}$. At each round $t$, the agent receives an observation characterized by $\cG$ in the sense that taking action $a\in\cA_t$, the observation $O_{t,a}$ contains the rewards $Y_{t,a'}=\theta^*_{a'}+\eta_{t,a'}$ for each action $a'\in N^{\text{out}}(a)$.
We review two standard quantities to describe the graph structure used in \citet{alon2015online}.

\begin{definition}[Independence number]
An independent set of a graph $\cG=(\cA, \cE)$ is a subset $\cC\subseteq \cA$ such that no two different $i,j\in\cA$
are connected by an edge in $\cE$. The cardinality of a largest independent set is the \emph{independence number} of $\cG$, denoted by $\beta(\cG)$.
\end{definition}

A vertex $a$ is observable if $N^{\text{in}}(a)\neq \emptyset$. A vertex is strongly observable if it has either a self-loop or incoming edges from all other vertices. A vertex is weakly observable if it is observable but not strongly observable. A graph $\cG$ is
observable if all its vertices are observable and it is strongly observable if all its vertices are
strongly observable. A graph is weakly observable if it is observable but not strongly observable.
\begin{definition}[Weak domination number]
In a directed graph $\cG=(\cA, \cE)$ with a set of weakly observable vertices $\cW\subseteq \cA$, a weakly dominating set $\cD\subseteq \cA$ is a set of vertices that dominates $\cW$. That means for any $w\in\cW$ there exists $d\in\cD$ such that $w\in N^{\text{out}}(d)$. The \emph{weak domination number} of $\cG$, denoted by $\delta(\cG)$, is the size of the smallest weakly dominating set.
\end{definition}

\subsection{Lower bound for conditional IDS}
We first prove that conditional IDS suffers $\Omega(\sqrt{n\beta(\cG)})$ Bayesian regret lower bound for a particular prior and show later
the optimal rate for this prior could be much smaller when $\beta(\cG)$ is very large (Section \ref{sec:upper_bound_graph}). 
\begin{theorem}\label{theorem:lower_bound_graph}
Let $\pi^{\text{CIR}}$ be a conditional IDS policy. There exists a contextual bandit with graph feedback instance such that $\beta(\cG)=k-1$ and
\begin{equation*}
    \BR(n; \pi^{\text{CIR}})\geq \frac{1}{16}\sqrt{n(\beta(\cG)-2)}\,.
\end{equation*}
\end{theorem}

\begin{proof} Let us construct the hard problem instance that generalizes Example 1 in Section \ref{sec:myopic}. Suppose $\{x_1, \ldots, x_{k}\}\subseteq \mathbb R^{k}$ is the set of basis vectors that corresponds to $k$ arms. The first $k-1$ arms form a standard multi-armed Gaussian bandit with unit variance. When $k$th arm is pulled, the agent always suffers a regret of 1, but obtains samples for the first $k-1$ arms. In terms of the language of graph feedback, the first $k-1$ arms only contain self-loop while the out-neighborhood of the last arm contains all the first $k-1$ arms. One can verify $\beta(\cG)=k-1$.

There are two context sets and the arrival probability of each context is 1/2.  Context set 1 consists of $\{x_{k-1}, x_{k}\}$ while context set 2 consists of $\{x_1,\ldots, x_{k-2}\}$. 

For each $i\in\{2, \ldots, k-2\}$, let $\theta^{(i)}\in\mathbb R^{k}$ with $\theta_j^{(i)}=\ind\{i=j\}\gamma$ for $j\in\{2, \ldots, k-2\}$ and $\theta^{(i)}_{1}=0$, $\theta^{(i)}_{k-1}=0$, $\theta_{k}^{(i)}=\gamma-1$ where $\gamma>0$ is the gap that will be chosen later. Assume the prior of $\theta^*$ is uniformly distributed over $\{\theta^{(2)}, \ldots, \theta^{(k-2)}\}$.

We write $\mathbb E_i[\cdot]=\mathbb E_{\theta^{(i)}}[\cdot]$ for short and the expectation is taken with respect to  the measures on the sequence
of outcomes $(A_1, Y_1, \ldots, A_n, Y_n)$ determined by $\theta^{(i)}$. Define the cumulative regret of policy $\pi$ interacting with bandit $\theta^{(i)}$ as
\begin{equation*}
\begin{split}
    R_{\theta^{(i)}}(n; \pi)
    &=\sum_{t=1}^n\mathbb E_i\left[\left(\langle x_{s_t}^*, \theta^{(i)}\rangle-Y_t\right)\ind(s_t=1)\right]\\
    &+\sum_{t=1}^n\mathbb E_i\left[\left(\langle x_{s_t}^*, \theta^{(i)}\rangle-Y_t\right)\ind(s_t=2)\right]\,,
    \end{split}
\end{equation*}
such that
$\BR(n; \pi)= \frac{1}{k-3}\sum_{i=2}^{k-2}R_{\theta^{(i)}}(n;\pi).
$

\paragraph{Step 1.} Fix a conditional IDS policy $\pi$. From the definition of conditional IDS in Eq.~\eqref{def:conditional_IDS}, when context set 1 is arriving, conditional IDS will always pull $x_{k-1}$ for this prior since $x_{k-1}$ is always the optimal arm for context set 1. This means conditional IDS will suffer no regret for context set 1, which implies for any $i\in \{2,\ldots,k-2\}$,
\begin{equation*}
\begin{split}
     &\sum_{t=1}^n\mathbb E_i\left[\left(\langle x_{s_t}^*, \theta^{(i)}\rangle-Y_t\right)\ind(s_t=1)\right]=0\,.
\end{split}
\end{equation*}

\paragraph{Step 2.} Define $T_i(n)$ as the number of pulls for arm $i$ over $n$ rounds.
 On the other hand,
\begin{equation*}
\begin{split}
&\mathbb E_i\left[ \sum_{t=1}^n \left(\langle x_{s_t}^*, \theta^{(i)}\rangle-Y_t\right)\ind(s_t=2)\right]\\
&=\mathbb E_i\left[ \sum_{j=1,j\neq i}^{k-2}T_j(n)\right]\gamma=\mathbb E_i\left[n_2-T_i(n)\right]\gamma\,,
\end{split}
\end{equation*}
where the first equation comes from the fact that the context sets 1 and 2 are disjoint and $n_2$ is the number of times that context 2 arrives over $n$ rounds. Since the context is generated independently with respect to the learning process, we have $\mathbb E_i[n_2]=n/2$.
By Pinsker's inequality and the divergence decomposition lemma \citep[Lemma 15.1]{lattimore2020bandit}, we know that with $\gamma = \frac{1}{2}\sqrt{k/n}$,
\begin{equation*}
\begin{split}
   \sum_{i=2}^{k-2} \mathbb E_i[T_i(n)]&\leq   \sum_{i=2}^{k-2} \mathbb E_1[T_i(n)]+\frac{1}{4} \sum_{i=2}^{k-2} \sqrt{nk\mathbb E_1(T_i(n))}\\
   &\leq n+\frac{1}{4}kn\,.
   \end{split}
\end{equation*}
Therefore,
\begin{equation*}
\begin{split}
     &\BR(n; \pi)
     = \frac{1}{k-3}\sum_{i=2}^{k-2}\mathbb E_i\left[n_2-T_i(n)\right]\gamma\\
     &=\frac{1}{k-3}\left(\frac{k-3}{2}n-\sum_{i=2}^{k-2}\mathbb E_i[T_i(n)]\right)\gamma \\
     &\geq \frac{1}{16}\sqrt{n(k-3)}\,.
     \end{split}
\end{equation*}
This finishes the proof.
\end{proof}
\begin{remark}
We can strengthen the lower bound such that it holds for any graph by replacing the standard multi-armed bandit instance in context set 2 by the hard instance in \citet[Theorem 5]{alon2017nonstochastic}.
\end{remark}

\subsection{Upper bound achieved by contextual IDS}\label{sec:upper_bound_graph} In this section, we derive a Bayesian regret upper bound achieved by contextual IDS. According to Theorem \ref{thm:generic_bound_ids}, it suffices to bound the worst-case MIR $\cI_{\alpha, \lambda}$ for $\lambda=2,3$ as well as the cumulative information gain respectively.  
Let $R_{\max}$ be an almost surely upper bound on the maximum expected reward.
\begin{lemma}[Squared information ratio]\label{lemma:graph_ir_2}
  For any $\epsilon\in[0, 1]$ and a strongly observable graph $\cG$, the $(2\epsilon R_{\max}, 2)$-MIR can be bounded by
    \begin{equation*}
      \cI_{2\epsilon R_{\max}, 2}\leq \frac{4R_{\max}^2+4}{1-\epsilon}\beta(\cG)\log\left(\frac{4k^2}{\beta(\cG)\epsilon}\right)\,.
    \end{equation*}
    \end{lemma}
The proof is deferred to Appendix \ref{proof:graph_ir_2} and the proof idea is to bound the worst-case squared information ratio by the squared information ratio of Thompson sampling.

Next we will bound $\cI_{\alpha, \lambda}$ for $\lambda=3$ in terms of an \emph{explorability constant}.
\begin{definition}[Explorability constant]
We define the explorability constant as $\vartheta(\cG, \xi) :=$
\begin{equation*}
     \max_{\pi:\cS\to \cP(\cA)}\min_{d\in\cA}\mathbb E_{s\sim \xi, a\sim\pi(\cdot|s)}\Big[\ind\left\{d\in N^{\text{out}}(a)\right\}\Big]\,.
\end{equation*}
\end{definition}

\begin{lemma}[Cubic information ratio]\label{lemma:graph_ir_3}
For any $\epsilon\in[0, 1]$ and a weakly observable graph $\cG$, the $(2\epsilon R_{\max}, 3)$-MIR can be bounded by 
\begin{equation*}
    \cI_{2\epsilon R_{\max},3}\leq \frac{R_{\max}^3+R_{\max}}{\vartheta(\cG, \xi)}\,.
\end{equation*}
\end{lemma}    
The proof is deferred to Appendix \ref{proof:graph_ir_3} and the proof idea is to bound the worst-case cubic information ratio by the cubic information ratio of a policy that mixes Thompson sampling and a policy maximizing the explorability constant.
    
\begin{lemma}\label{lemma:cumulative_ig}
The cumulative information gain can be bounded by
    \begin{equation*}
     \mathbb E\left[\sum_{t=1}^n\mathbb E_t\left[\mI_t(\pi^*)^{\top}\pi_t(\cdot|s_t)\right]\right] \leq \mathbb H(\pi^*)\leq M\log(k)\,,
\end{equation*}
where $\mathbb H(\cdot)$ is the entropy and $M$ is the number of available contexts. 
\end{lemma}   
The proof is deferred to Appendix \ref{proof:cumulative_ig}. Combining the results in Lemmas \ref{lemma:graph_ir_2}-\ref{lemma:cumulative_ig} and the generic regret bound in Theorem \ref{thm:generic_bound_ids}, we obtain the follow theorem.
\begin{theorem}[Regret bound for graph contextual IDS]\label{thm:main_graph}
Suppose $\pi^{\text{MIR}} = (\pi_{t})_{t\in \mathbb N}$ where $\pi_{t}$ minimizes $(2R_{\max}/\sqrt{n}, 2)$-MIR at each round. If $\cG$ is strongly observable, we have 
\begin{equation*}\label{eqn:upper_bound_graph}
     \begin{split}
        &\BR(n; \pi^{\text{MIR}})\leq  CR_{\max}\min\Bigg( \left(\frac{2M\log(k)}{\vartheta(\cG, \xi)}\right)^{\tfrac{1}{3}}n^{\tfrac{2}{3}},\\
        &\sqrt{\beta(\cG)\log\left(\frac{4k^2\sqrt{n}}{\beta(\cG)}\right)nM\log(k)}\Bigg)\,,
     \end{split}
 \end{equation*}
 where $C$ is an absolute constant.
\end{theorem}
Ignoring the constant and logarithmic terms, the Bayesian regret upper bound in Theorem \ref{thm:main_graph} can be simplified to 
$$
\tilde{O}\left(\min\left(\sqrt{\beta(\cG)Mn}, (M/\vartheta(\cG, \xi))^{1/3}n^{2/3}\right)\right)\,.
$$ 
When the independence number is large enough such that $\beta(\cG)\geq (\vartheta(\cG, \xi))^{-2/3} (n/M)^{1/3}$, this bound is dominated by $\tilde{O}((M/\vartheta(\cG, \xi))^{1/3}n^{2/3})$ that is independent of the independence number. Together with Theorem \ref{theorem:lower_bound_graph}, we can argue that conditional IDS is sub-optimal comparing with contextual IDS in this regime. 

\begin{remark}[Connection between $\vartheta(\cG, \xi)$ and $\delta(\cG)$]
Let $\cD\subseteq \cA$ be the smallest weakly dominating set of the full graph with $|\cD|=\delta(\cG)$. Consider a policy $\mu$ such that $\mu(\cdot|s_t)$ is uniform over $\cD\cap \cA_t$. Then we have 
\begin{equation*}
    \min_{d\in\cA}\mathbb E_{s\sim \xi, a\sim\mu(\cdot|s)}\left[\ind\left\{d\in N^{\text{out}}(a)\right\}\right]\geq \frac{1}{M\delta(\cG)}\,,
\end{equation*}
which implies $1/\vartheta(\cG, \xi) \leq M\delta(\cG)$. When $M=1$, \citet{alon2015online} demonstrated that the minimax optimal rate for weakly observable graph is $\tilde{\Theta}(\delta(\cG)^{1/3}n^{2/3})$ which implies $1/\vartheta(\cG)\gtrsim\delta(\cG)$ up to some logarithm factors.
\end{remark}

\begin{remark}[Open problem]
For the fixed action set setting, \citet{alon2015online} proved that the independence number and weakly domination number are the fundamentally quantity to describe the graph structure. However, our regret upper bound has the number of contexts $M$ appearing. It is interesting to investigate if $M$ can be removed for contextual bandits with graph feedback.

Note that one can also bound $\mathbb H(\pi^*)\leq \mathbb H(\theta^*)\leq k$ such that the dependency on $M$ does not appear. But it is unclear the right dependency on $M$ and $k$ for the optimal regret rate in the contextual setting.
\end{remark}

%% file: sparse_bandits.tex
\section{Sparse Linear Contextual Bandits}\label{sec:sparse_bandits}
We consider a Bayesian formulation of sparse linear contextual bandits.
The notion of sparsity can be defined through the parameter space $\Theta$:
\begin{equation*}
    \Theta = \left\{\theta\in\mathbb R^d \Bigg| \sum_{j=1}^d\ind\{\theta_j\neq 0\}\leq s, \|\theta\|_2\leq 1\right\}\,.
\end{equation*}
We assume $s$ is known and it can be relaxed by putting a prior on it. We consider the case where $\cA^m=\cA$ for all $m\in\cS$. Suppose $\phi:\cS\times \cA \to \mathbb R^d$ is a feature map known to the agent.  In sparse linear contextual bandits, the reward of taking action $a$ is 
$f(s_t, a, \theta^*)  = \phi(s_t, a)^{\top}\theta^*+\eta_{t,a}\,,
$
where $\theta^*$ is a random variable taking values in $\Theta$ and denote $\rho$ as the prior distribution. 

We first define a quantity to describe the explorability of the feature set. 
\begin{definition}[Explorability constant]\label{def:explortory}
We define the explorability constant as $C_{\min}(\phi, \xi) :=$
\begin{equation*}
     \max_{\pi:\cS\to \cP(\cA)} \sigma_{\min}\left(\mathbb E_{s\sim \xi}\Big[\mathbb E_{a\sim \pi(\cdot|s)}[\phi(s, a)\phi(s, a)^{\top}]\Big]\right)\,.
\end{equation*}
\end{definition}
\begin{remark}
The explorability constant is a problem-dependent constant that has previously been introduced in \citet{hao2020high, hao2021information} for a non-contextual sparse linear bandits. For an easy instance when $\{\phi(s_t, a)\}_{a\in\cA}$ is a full hypercube for every $s_t\in\cS$, $C_{\min}(\phi, \xi)$ could be as small as 1 where the policy is to sample uniformly from the corner of the hypercube no matter which context arrives.
\end{remark}
\subsection{Lower bound for conditional IDS}
We prove an algorithm-dependent lower bound for a sparse linear contextual bandits instance.
\begin{theorem}\label{thm:lower_bound_sparse}
Let $\pi^{\text{CIR}}$ be a conditional IDS policy. There exists a sparse linear contextual bandit instance whose explorability constant is 1/2 such that
\begin{equation*}
    \BR(n; \pi^{\text{CIR}})\geq \frac{1}{16}\sqrt{nds}\,.
\end{equation*}
\end{theorem}

\paragraph{How about the principle of optimism?} \citet[Section 4]{hao2021information} has shown that even for non-contextual case, optimism-based algorithms such as UCB or Thompson sampling fail to optimally balance information and regret and result in a sub-optimal regret bound.

\subsection{Upper bound achieved by contextual IDS} 
In this section, we will prove that contextual IDS can achieve a nearly dimension-free regret bound.
We say that the feature set has sparse optimal actions if the optimal action is $s$-sparse for each context almost surely with respect to the prior.

\begin{lemma}[Bounds for information ratio]\label{lemma:information_ratio_sparse}
We have $ \cI_{0,2}\leq d/2$. When the feature set has sparse optimal actions, we have $\cI_{0,3}\leq s^2/(4C_{\min}(\phi, \xi) ).
$
\end{lemma}

From the definition of $\pi^*(m) = \argmin_{a\in\cA^m} a^{\top}\theta^*$, we know that $\pi^*$ is a deterministic function of $\theta^*$. By the data processing lemma, we have $\mI(\pi^*;\cF_{n+1})\leq \mI(\theta^*;\cF_{n+1})$. According to Lemma 5.8 in \citet{hao2021information}, we have the following lemma.
\begin{lemma}\label{lemma:cum_sparse}
The cumulative information gain can be bounded by
 \begin{equation*}
     \mathbb E\left[\sum_{t=1}^n\mathbb E_t\left[\mI_t(\pi^*)^{\top}\pi_t(\cdot|s_t)\right]\right] \leq 2s\log(dn^{1/2}/s)\,.
\end{equation*}
\end{lemma}

Combining the results in Lemmas \ref{lemma:information_ratio_sparse}-\ref{lemma:cum_sparse} with the generic regret bound in  Theorem \ref{thm:generic_bound_ids}, we obtain the following regret bound.
\begin{theorem}[Regret bound for sparse contextual IDS]\label{thm:IDS}
Suppose $\pi^{\text{MIR}} = (\pi_{t})_{t\in \mathbb N}$ where $\pi_{t} = \argmin_{\pi}\Psi_{t,0}^{2}(\pi)$. 
When the feature set has sparse optimal actions, the following regret bound holds
\begin{equation*}
\begin{split}
     &\BR(n;\pi^{\text{MIR}}) \lesssim\\ &\min\Big\{ \sqrt{nds\log(dn^{1/2}/s)},\frac{sn^{\frac{2}{3}}(\log(dn^{1/2}/s))^{\tfrac{1}{3}}}{(C_{\min}(\phi, \xi) )^{\frac{1}{3}}}\Big\}\,.
\end{split}
\end{equation*}
\end{theorem}
In the data-poor regime where $d\gtrsim n^{1/3}s/C_{\min}^{2/3}$, contextual IDS achieves $\tilde{O}(sn^{2/3})$ regret bound that is tighter than the lower bound for conditional IDS. This rate matches the minimax lower bound derived in \citet{hao2020high} up to a $s$ factor in the data-poor regime.

%% file: algorithm.tex
\section{Practical Algorithm}
As shown in Section \ref{sec:conditional_IDS}, Conditional IDS only needs to optimize over probability distributions.
As proved by \citet[Proposition 6]{russo2018learning}, conditional IDS suffices to randomize over two actions at each round and the computational complexity is further improved to $O(|\cA|)$ by \citet[Lemma 2.7]{kirschner2021information}. However, contextual IDS in Section \ref{sec:contextual_IDS} needs to optimize over a mapping from the context space to action space at each round which in general is much more computationally demanding.

To obtain a practical and scalable algorithm, we approximate the contextual IDS using Actor-Critic \citep{konda2000actor}. We parametrize the policy class by a neural network and optimize the information ratio by multi-steps stochastic gradient decent. 

Consider a parametrized policy class $\Pi_{\theta}=\{\pi_{\theta}:\cS\to \cP(\cA)\}$. The policy, which is a conditional probability distribution $\pi_{\theta}$, can be parameterized with a neural network. This
neural network maps (deterministically) from a context $s$ to a probability distribution over $\cA$. We further parametrize the critic which is the reward function by a value network $Q_{\theta}$.

To avoid additional approximation errors, we assume we can sample from the true context distribution. 
At each round, the parametrized contextual IDS minimizes the following empirical MIR loss through SGD: 
\begin{equation}\label{def:MIR}
   \min_{\pi\in\Pi_{\theta}}  \frac{ \sum_{i=1}^{w}[\Delta_t(s_t^{(i)})^{\top}\pi(\cdot|s_t^{(i)})]^{2}}{\sum_{i=1}^w[\mI_t(\pi^*)^{\top}\pi(\cdot|s_t^{(i)})]}\,,
\end{equation}
where $\{s_t^{(1)},\ldots, s_t^{(w)}\}$ are the independent samples of contexts. We use the Epistemic Neural Networks (ENN) \citep{osband2021epistemic} to quantify the posterior uncertainty of the value network. For each given $s_t^{(i)}$, following the procedure in Section 6.3.3 and 6.3.4 of \citet{lu2021reinforcement}, we can approximate the one-step expected regret and the information gain efficiently using the samples outputted by ENN.

%% file: experiment.tex
\section{Experiment}
We conduct some preliminary experiments to evaluate the parametrized contextual IDS through a neural network contextual bandit. 

At each round $t$, the environment independently generates an observation in the form of $d$-dimensional contextual vector $x_t$ from some distributions. Let $f_{\theta^*}: \mathbb R^d \to \mathbb R^{|\cA|}$ be a neural network link function. When the agent takes action $a$, she will receive a reward in the form of
$Y_{t,a} = [f_{\theta^*}(x_t)]_a+\eta_{t,a}.$
This is the not sharing parameter formulation for contextual bandits which means each arm has its own parameter. 

We set the generative model $f_{\theta^*}$ being a 2-hidden-layer ReLU MLP with 10 hidden neurons. The number of actions is 5. The contextual vector $x_t\in\mathbb R^{10}$ is sampled from $N(0, I_{10})$ and the noise is sampled from standard Gaussian distribution. 

We compare contextual IDS with conditional IDS and Thompson sampling. For a fair comparison, we use the same ENN architecture to obtain posterior samples for three agents. As reported by \citet{osband2021evaluating}, ensemble sampling with randomized prior function tends to be the best ENN that balances the computation and accuracy so we use 10 ensembles in our experiment. With 200 posterior samples, we use the same way described by \citet{lu2021reinforcement} to approximate the one-step regret and information gain for both conditional IDS and contextual IDS. We sample 20 independent contextual vectors at each round. Both the policy network and value network are using 2-hidden-layer ReLU MLP with 20 hidden neurons and optimized by Adam with learning rate 0.001. We report our results in Figure \ref{fig:nn} where parametrized contextual IDS achieves reasonably good regret.
    \begin{figure}
 \centering
 \includegraphics[width=0.8\linewidth]{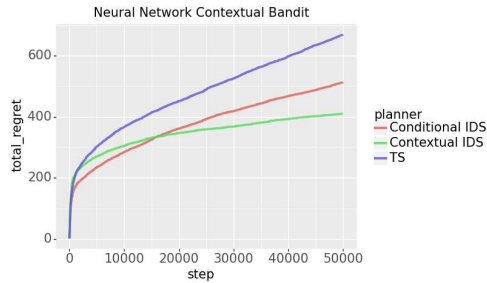}
\caption{Cumulative regret for a neural network contextual bandit.}
\label{fig:nn}
\end{figure}

%% file: discussion.tex
\section{Discussion and Future Work}
In this work, we investigate the right form of information ratio for contextual bandits and emphasize the importance of utilizing context distribution information through contextual bandits with graph feedback and sparse linear contextual bandits.
For linear contextual bandits with moderately small number of actions, one future work is to see if contextual IDS can achieve $O(\sqrt{dn\log(k)})$ regret bound.

%% file: appendix.tex
 \section{General Results}
 \subsection{Proof of Theorem \ref{thm:generic_bound_ids}}\label{sec:proof_generic_bound}
 \begin{proof}
 From the definitions in Eqs.~\eqref{eqn:bayesian_regret} and \eqref{def:one_step_regret}, we could decompose the Bayesian regret bound of contextual IDS as
 \begin{equation}\label{eqn:bayesian_regret_decom}
     \BR(n; \pi^{\text{MIR}}) =n\alpha+ \mathbb E\left[\sum_{t=1}^n\mathbb E_{s_t}\left[(\Delta_t(s_t)-\alpha)^{\top}\pi_t(\cdot|s_t)\right]\right]\,.
 \end{equation}
 Recall that at each round contextual IDS follows
 \begin{equation*}
     \pi_t = \argmin_{\pi:\cS\to \cP(\cA)} \Psi_{t, \alpha}^{2}(\pi) =  \argmin_{\pi:\cS\to \cP(\cA)}\frac{ \max\left(0, (\mathbb E_{s_t}[\Delta_t(s_t)^{\top}\pi(\cdot|s_t)]-\alpha)\right)^{2}}{\mathbb E_{s_t}[\mI_t(\pi^*)^{\top}\pi(\cdot|s_t)]}\,.
 \end{equation*}
 Moreover, we define
 \begin{equation*}
     q_{t, \lambda}=\argmin_{\pi:\cS\to \cP(\cA)} \Psi_{t, \alpha}^{\lambda}(\pi) =  \argmin_{\pi:\cS\to \cP(\cA)}\frac{ \max\left(0, (\mathbb E_{s_t}[\Delta_t(s_t)^{\top}\pi(\cdot|s_t)]-\alpha)\right)^{\lambda}}{\mathbb E_{s_t}[\mI_t(\pi^*)^{\top}\pi(\cdot|s_t)]}\,.
 \end{equation*}
 Suppose for a moment that $\mathbb E_{s_t}[\Delta_t(s_t)^{\top}\pi_t(\cdot|s_t)]-\alpha>0$. Let $M$ be the number of contexts. Then the derivative can be written as 
 \begin{equation*}
     \nabla_{\pi(\cdot|m)}\Psi_{t,\alpha}^{2}(\pi_t) = \frac{2p_m\Delta_t(m)(\mathbb E_{s_t}[\Delta_t(s_t)^{\top}\pi(\cdot|s_t)]-\alpha)}{\mathbb E_{s_t}[\mI_t(\pi^*)^{\top}\pi(\cdot|s_t)]}-\frac{p_m\mI_t(m)\left(\mathbb E_{s_t}[\Delta_t(s_t)^{\top}\pi(\cdot|s_t)]-\alpha)\right)^2}{\left(\mathbb E_{s_t}[\mI_t(\pi^*)^{\top}\pi(\cdot|s_t)]\right)^2}\,,
 \end{equation*}
 where $p_m$ is the arrival probability for context $m$. Using the first-order optimality condition for each $m\in[M]$,
 \begin{equation*}
     \mathbb E_{s_t}\left[\langle \nabla_{\pi(\cdot|s_t)}\Psi_{t,2}(\pi_t), q_{t,\lambda}(\cdot|s_t)-\pi_t(\cdot|s_t)\rangle\right]=\sum_{m=1}^M p_m \langle \nabla_{\pi(\cdot|m)}\Psi_{t,\alpha}^2(\pi_t), q_{t,\lambda}(\cdot|m)-\pi_t(\cdot|m)\rangle\geq 0\,.
 \end{equation*}
 This implies 
 \begin{equation*}
 \begin{split}
&\frac{2\mathbb E_{s_t}\left[\Delta_t(s_t)^{\top}(q_{t,\lambda}(\cdot|s_t)-\pi_t(\cdot|s_t))\right](\mathbb E_{s_t}[\Delta_t(s_t)^{\top}\pi_t(\cdot|s_t)]-\alpha)}{\mathbb E_{s_t}[\mI_t(\pi^*)^{\top}\pi_t(\cdot|s_t)]}\\
&\geq \frac{\mathbb E_{s_t}\left[\mI_t(\pi^*)^{\top}(q_{t,\lambda}(\cdot|s_t)-\pi_t(\cdot|s_t))\right]\left(\mathbb E_{s_t}[\Delta_t(s_t)^{\top}\pi_t(\cdot|s_t)]-\alpha)\right)^2}{\left(\mathbb E_{s_t}[\mI_t(\pi^*)^{\top}\pi_t(\cdot|s_t)]\right)^2} \,.
 \end{split}
 \end{equation*}
 Since we assume $\mathbb E_{s_t}[\Delta_t(s_t)^{\top}(\pi_t(\cdot|s_t)]-\alpha>0$ and the information gain is always non-negative, we have 
 \begin{equation*}
     \begin{split}
     2\mathbb E_{s_t}\left[\Delta_t(s_t)^{\top}(q_{t,\lambda}(\cdot|s_t)-\pi_t(\cdot|s_t))\right]\geq \frac{\mathbb E_{s_t}\left[\mI_t(\pi^*)^{\top}(q_{t,\lambda}(\cdot|s_t)-\pi_t(\cdot|s_t))\right]\mathbb E_{s_t}[\Delta_t(s_t)^{\top}\pi_t(\cdot|s_t)]-\alpha)}{\mathbb E_{s_t}[\mI_t(\pi^*)^{\top}\pi_t(\cdot|s_t)]}\,.
     \end{split}
 \end{equation*}
 which implies
  \begin{equation*}
     \begin{split}
     2\left(\mathbb E_{s_t}[\Delta_t(s_t)^{\top}q_{t,\lambda}(\cdot|s_t)]-\alpha\right)&\geq \left(\mathbb E_{s_t}[\Delta_t(s_t)^{\top}\pi_t(\cdot|s_t)]-\alpha\right)\left(1+\frac{\mathbb E_{s_t}[\mI_t(\pi^*)^{\top}q_{t,\lambda}(\cdot|s_t)]}{\mathbb E_{s_t}[\mI_t(\pi^*)^{\top}\pi_t(\cdot|s_t)]}\right)\\
     &\geq\mathbb E_{s_t}[\Delta_t(s_t)^{\top}\pi_t(\cdot|s_t)]-\alpha\,.
     \end{split}
 \end{equation*}
 
 Putting the above results together,
 \begin{equation*}
     \begin{split}
         \frac{(\mathbb E_{s_t}[\Delta_t(s_t)^{\top}\pi_t(\cdot|s_t)]-\alpha)^{\lambda}}{\mathbb E_{s_t}[\mI_t(\pi^*)^{\top}\pi(\cdot|s_t)]} &=  \frac{(\mathbb E_{s_t}[\Delta_t(s_t)^{\top}\pi_t(\cdot|s_t)]-\alpha)^{2}(\mathbb E_{s_t}[\Delta_t(s_t)^{\top}\pi(\cdot|s_t)]-\alpha)^{\lambda-2}}{\mathbb E_{s_t}[\mI_t(\pi^*)^{\top}\pi(\cdot|s_t)]} \\
         &\leq \frac{2^{\lambda-2}(\mathbb E_{s_t}[\Delta_t(s_t)^{\top}\pi_t(\cdot|s_t)]-\alpha)^{2}(\mathbb E_{s_t}[\Delta_t(s_t)^{\top}q_{t,\lambda}(\cdot|s_t)]-\alpha)^{\lambda-2}}{\mathbb E_{s_t}[\mI_t(\pi^*)^{\top}\pi_t(\cdot|s_t)]}\\
         & \leq \frac{2^{\lambda-2}(\mathbb E_{s_t}[\Delta_t(s_t)^{\top}q_{t,\lambda}(\cdot|s_t)]-\alpha)^{2}(\mathbb E_{s_t}[\Delta_t(s_t)^{\top}q_{t,\lambda}(\cdot|s_t)]-\alpha)^{\lambda-2}}{\mathbb E_{s_t}[\mI_t(\pi^*)^{\top}q_{t,\lambda}(\cdot|s_t)]}\\
         &= \frac{2^{\lambda-2}(\mathbb E_{s_t}[\Delta_t(s_t)^{\top}q_{t,\lambda}(\cdot|s_t)]-\alpha)^{\lambda}}{\mathbb E_{s_t}[\mI_t(\pi^*)^{\top}q_{t,\lambda}(\cdot|s_t)]}\,.
     \end{split}
 \end{equation*}
Recall that $\cI_{\alpha,\lambda}$ is the worst-case information ratio such that for any $t\in[n]$, $\Psi_{t,\alpha}^{\lambda}(\pi_t)\leq \cI_{\alpha,\lambda}$ almost surely. Therefore,
 \begin{equation}\label{eqn:decomp}
     \mathbb E_{s_t}\left[\Delta_t(s_t)^{\top}\pi_t(\cdot|s_t)-\alpha\right]\leq 2^{1-2/\lambda}\left(\mathbb E_{s_t}\left[\mI_t(\pi^*)^{\top}\pi_t(\cdot|s_t)\right]\right)^{1/\lambda}\cI_{\alpha,\lambda}^{1/\lambda}\,,
 \end{equation}
  which is obvious when $\mathbb E_{s_t}[\Delta_t(s_t)^{\top}(\pi_t(\cdot|s_t)]-\alpha\leq 0$.
Combining Eqs.~\eqref{eqn:bayesian_regret_decom} and \eqref{eqn:decomp} together and using Holder's inequality, we obtain
 \begin{equation*}
      \BR(n; \pi^{\text{IDS}})\leq n\alpha+2^{1-2/\lambda}\cI_{\alpha,\lambda}^{1/\lambda}n^{1-1/\lambda}\mathbb E\left[\sum_{t=1}^n\mathbb E_{s_t}\left[\mI_t(\pi^*)^{\top}\pi_t(\cdot|s_t)\right]\right]^{1/\lambda}\,.
 \end{equation*}
 This ends the proof.
\end{proof} 

\section{Contextual Bandits with Graph Feedback} 

\subsection{Proof of Lemma \ref{lemma:graph_ir_2}}\label{proof:graph_ir_2}
 \begin{proof}
We bound the squared information ratio in terms of independence number $\beta(\cG)$. Let $\pi_t^{\TS}(\cdot|s_t) = \mathbb P_t(a_t^*=\cdot)$ and consider a mixture policy $\pi_t^{\mix} = (1-\epsilon)\pi_t^{\TS}+\epsilon/k$ with some mixing parameter $\epsilon\in[0, 1]$ that will be determined later. 

First, we will derive an upper bound for one-step regret. From the definition of one-step regret,
 \begin{equation}\label{eqn:one_step_regret}
 \begin{split}
      \Delta_t(s_t)^{\top}\pi_t^{\mix}(\cdot|s_t) &= \sum_{a\in\cA_t}\pi_t^{\mix}(a|s_t)\mathbb E_t\left[Y_{t, a_t^*}-Y_{t, a}|s_t\right]\\
      &=(1-\epsilon)\sum_{a\in\cA_t}\pi_t^{\TS}(a|s_t)\mathbb E_t\left[Y_{t, a_t^*}-Y_{t, a}|s_t\right]+\epsilon\sum_{a\in\cA_t}\frac{1}{k}\mathbb E_t\left[Y_{t, a_t^*}-Y_{t, a}|s_t\right]\,.
 \end{split}
 \end{equation}
Recall that $R_{\max}$ is the almost surely upper bound of maximum expected reward and 
 let $$
 d_t(a, a') = D_{\KL}(\mathbb P_t(Y_{t,a}\in \cdot| a_t^*=a'||\mathbb P_t(Y_{t,a}\in\cdot))).
 $$ It is easy to see $Y_{t,a}$ is a $\sqrt{R_{\max}^2+1}$ sub-Gaussian random variable.
 For the first term of Eq.~\eqref{eqn:one_step_regret}, according to Lemma 3 in \citet{russo2014learning}, we have 
 \begin{equation}\label{eqn:bound1_graph}
     \sum_{a\in\cA_t}\pi_t^{\TS}(a|s_t)\mathbb E_t\left[Y_{t, a_t^*}-Y_{t, a}|s_t\right]\leq \sum_{a\in\cA_t}\pi_t^{\TS}(a|s_t)\sqrt{\frac{R_{\max}^2+1}{2}d_t(a, a)}\,.
 \end{equation}
 For the second term of Eq.~\eqref{eqn:one_step_regret}, we directly bound it by
 \begin{equation}\label{eqn:bound2_graph}
    \epsilon\sum_{a\in\cA_t}\frac{1}{k}\mathbb E_t\left[Y_{t, a_t^*}-Y_{t, a}|s_t\right]\leq 2\epsilon R_{\max}\,.
 \end{equation}
Putting Eqs.~\eqref{eqn:one_step_regret}-\eqref{eqn:bound2_graph} together, we have
\begin{equation*}
     \Delta_t(s_t)^{\top}\pi_t^{\mix}(\cdot|s_t)-2\epsilon R_{\max} \leq \sum_{a\in\cA_t}\pi_t^{\TS}(a|s_t)\sqrt{\frac{R_{\max}^2+1}{2}d_t(a, a)}\,. 
\end{equation*}
 
Second, we will derive a lower bound for the information gain. From data processing inequality and conditional on $s_t$, 
\begin{equation}\label{eqn:data_processing}
     \mI_t\left(\pi^*; (Y_{t,a'})_{a'\in N^{\text{out}}(a)}\right)\geq  \mI_t\left(a_t^*; (Y_{t,a'})_{a'\in N^{\text{out}}(a)}\right)\,.
\end{equation}
Let $E = (a_{ij})_{1\leq i,j \leq |\cA|}$ be the adjacency matrix that represents the graph feedback structure $\cG$. Then $a_{ij}=1$ if there exists an edge $(i, j)\in \cE$ and $a_{ij}=0$ otherwise.  Since for any $a_i, a_j\in N^{\text{out}}(a)$, $Y_{t,a_i}$ and $Y_{t, a_j}$ are mutually independent, 
\begin{equation*}
    \mI_t\left(a_t^*; (Y_{t,a'})_{a'\in N^{\text{out}}(a)}\right) \geq \sum_{a'\in N^{\text{out}}(a)}\mI_t(a_t^*; Y_{t,a'}) = \sum_{a'\in\cA_t}E(a, a')\mI(a_t^*; Y_{t,a'})\,.
\end{equation*}
Therefore,
 \begin{equation*}
 \begin{split}
      \mI_t(a_t^*)^{\top}\pi_t^{\mix}(\cdot|s_t)&\geq \sum_{a\in\cA_t} \mI_t\left(a_t^*; (Y_{t,a'})_{a'\in N^{\text{out}}(a)}\right)\pi_t^{\mix}(a|s_t)\\
      &\geq \sum_{a\in\cA_t}  \sum_{a'\in\cA_t}E(a, a')\mI_t(a_t^*; Y_{t,a'})\pi_t^{\mix}(a|s_t)\\
      &=\sum_{a\in\cA_t}\left(\pi_t^{\mix}(a|s_t)+\sum_{a'\in N^{\text{in}}(a)}\pi_t^{\mix}(a'|s_t)\right)\mI_t(a_t^*; Y_{t,a})\\
      &= \sum_{a\in\cA_t}\left(\pi_t^{\mix}(a|s_t)+\sum_{a'\in N^{\text{in}}(a)}\pi_t^{\mix}(a'|s_t)\right)\sum_{a'\in\cA_t}\pi_t^{\TS}(a'|s_t)d_t(a, a')\\
      &\geq \sum_{a\in\cA_t}\left(\pi_t^{\mix}(a|s_t)+\sum_{a'\in N^{\text{in}}(a)}\pi_t^{\mix}(a'|s_t)\right)\pi_t^{\TS}(a|s_t)d_t(a, a)\\
      &=\sum_{a\in\cA_t}\frac{\pi_t^{\mix}(a|s_t)+\sum_{a'\in N^{\text{in}}(a)}\pi_t^{\mix}(a'|s_t)}{\pi_t^{\TS}(a|s_t)}\pi_t^{\TS}(a|s_t)^2d_t(a,a)\,.
 \end{split}
 \end{equation*}
   Let's denote 
 \begin{equation*}
     U_{a,t}=\frac{\pi_t^{\mix}(a|s_t)+\sum_{a'\in N_{s_t}^{\text{in}}(a)}\pi_t^{\mix}(a'|s_t)}{\pi_t^{\TS}(a|s_t)}\,.
 \end{equation*}
 Putting the above together,
 \begin{equation*}
     \begin{split}
     \Delta_t(s_t)^{\top}\pi_t^{\mix}(\cdot|s_t)-2\epsilon R_{\max}&\leq \sqrt{\frac{R_{\max}^2+1}{2}}\sqrt{\sum_{a\in\cA_t}\frac{1}{ U_{a,t}}}\sqrt{\sum_{a\in\cA_t} U_{a,t}\pi_t^{\TS}(a|s_t)^2d_t(a,a)} \\
     &\leq \sqrt{\frac{R_{\max}^2+1}{2}}\sqrt{\sum_{a\in\cA_t}\frac{1}{ U_{a,t}}}\sqrt{\mI_t(a_t^*)^{\top}\pi_t^{\mix}(\cdot|s_t)}\,,
     \end{split}
 \end{equation*}
 where the first inequality is due to Cauchy–Schwarz inequality. This implies
 \begin{equation*}
     \begin{split}
         \frac{\left(\Delta_t(s_t)^{\top}\pi_t^{\mix}(\cdot|s_t)-2\epsilon R_{\max} \right)^2}{\mI_t(a_t^*)^{\top}\pi_t^{\mix}(\cdot|s_t)}&\leq \frac{R_{\max}^2+1}{2}\sum_{a\in\cA_t}\frac{\pi_t^{\TS}(a|s_t)}{\pi_t^{\mix}(a|s_t)+\sum_{a'\in N^{\text{in}}(a)}\pi_t^{\mix}(a'|s_t)}\\
         &\leq \frac{R_{\max}^2+1}{2(1-\epsilon)}\sum_{a\in\cA_t}\frac{\pi_t^{\mix}(a|s_t)}{\pi_t^{\mix}(a|s_t)+\sum_{a'\in N^{\text{in}}(a)}\pi_t^{\mix}(a'|s_t)}\,,
     \end{split}
 \end{equation*}
 where we use $\pi_t^{\mix}\geq (1-\epsilon)\pi_t^{\TS}$. Using Jenson's inequality and Eq.~\eqref{eqn:data_processing},
 \begin{equation*}
 \begin{split}
      \frac{(\mathbb E_{s_t}[\Delta_t(s_t)^{\top}\pi^{\mix}_t(\cdot|s_t)]-2\epsilon R_{\max})^2}{\mathbb E_{s_t}[\mI_t(\pi^*)^{\top}\pi_t^{\mix}(\cdot|s_t)] }&\leq\frac{(\mathbb E_{s_t}[\Delta_t(s_t)^{\top}\pi^{\mix}_t(\cdot|s_t)]-2\epsilon R_{\max})^2}{\mathbb E_{s_t}[\mI_t(a_t^*)^{\top}\pi_t^{\mix}(\cdot|s_t)] }\\
      &\leq \mathbb E_{s_t}\left[\frac{(\Delta_t(s_t)^{\top}\pi_t^{\mix}(\cdot|s_t)-\alpha)^2}{\mI_t(a_t^*)^{\top}\pi_t^{\mix}(\cdot|s_t) }\right]\\
     &\leq \mathbb E_{s_t}\left[\frac{R_{\max}^2+1}{2(1-\epsilon)}\sum_{a\in\cA_t}\frac{\pi_t^{\mix}(a|s_t)}{\pi_t^{\mix}(a|s_t)+\sum_{a'\in N^{\text{in}}(a)}\pi_t^{\mix}(a'|s_t)}\right]\,.
      \end{split}
 \end{equation*}
 
Next we restate Lemma 5 from \citet{alon2015online} to bound the right hand side.
\begin{lemma}\label{lemma:independence_number}
Let $\cG=(\cA, \cE)$ be a directed graph with $|\cA|=k$. Assume a distribution $\pi$ over $\cA$ such that $\pi(i)\geq \eta$ for all $i\in\cA$ with some constant $0<\eta<0.5$. Then we have
\begin{equation*}
    \sum_{i\in\cA}\frac{\pi(i)}{\pi(i)+\sum_{j\in N^{\text{in}}(i)}\pi(j)} \leq 4\beta(G)\log\left(\frac{4k}{\beta(\cG)\eta}\right)\,.
\end{equation*}
\end{lemma} 
 
Using Lemma \ref{lemma:independence_number} with $\eta=\epsilon/k$,
 \begin{equation*}
 \begin{split}
    \frac{(\mathbb E_{s_t}[\Delta_t(s_t)^{\top}\pi^{\mix}_t(\cdot|s_t)]-2\epsilon R_{\max})^2}{\mathbb E_{s_t}[\mI_t(\pi^*)^{\top}\pi_t^{\mix}(\cdot|s_t)] }&\leq \frac{2R_{\max}^2+2}{1-\epsilon}2\beta(\cG)\log\left(\frac{4k^2}{\beta(\cG)\alpha}\right)\,.
     \end{split}
 \end{equation*}
According to the definition of contextual IDS, this ends the proof.
\end{proof}

\subsection{Proof of Lemma \ref{lemma:graph_ir_3}}\label{proof:graph_ir_3}
\begin{proof}
We bound the worst-case marginal information ratio with $\lambda=3$. Define an exploratory policy $\mu:\cS\to \cA$ such that 
\begin{equation*}
    \mu = \argmax_{\pi}\min_d\mathbb E_{s, a\sim\pi(\cdot|s)}\left[\ind\left\{d\in N^{\text{out}}(a)\right\}\right]\,.
\end{equation*}
Consider a mixture policy $\pi_t^{\mix} = (1-\epsilon)\pi_t^{\TS}+\epsilon\mu$ for some mixing parameter $\epsilon\in[0, 1]$. Since for any $a_i, a_j\in N^{\text{out}}(a)$, $Y_{t,a_i}$ and $Y_{t, a_j}$ are mutually independent, we have
 \begin{equation*}
 \begin{split}
     \mathbb E_{s_t}\left[\mI_t(\pi^*)^{\top}\pi_t^{\mix}(\cdot|s_t)\right]&= \mathbb E_{s_t, a\sim\pi_t^{\mix}(\cdot|s_t)}\left[\mI_t\left(\pi^*;(Y_{t,a'})_{a'\in  N^{\text{out}}(a)}\right)\right]\\
     &\geq \mathbb E_{s_t, a\sim\pi_t^{\mix}(\cdot|s_t)}\left[\mI_t\left(a_t^*;(Y_{t,a'})_{a'\in  N^{\text{out}}(a)}\right)\right]\\
     &\geq
     \mathbb E_{s_t, a\sim\pi_t^{\mix}(\cdot|s_t)}\left[\sum_{a'\in N^{\text{out}}(a)}\mI_t\left(a_t^*;Y_{t,a'}\right)\right]\,,
\end{split}
\end{equation*}
where the first inequality is from data processing inequality.
From the definition of the mixture policy, $\pi^{\text{mix}}(a|s_t)\geq \epsilon \mu(a|s_t)$ for any $a\in\cA_t$. Then we have
\begin{equation*}
    \begin{split}
    \mathbb E_{s_t}\left[\mI_t(\pi^*)^{\top}\pi_t^{\mix}(\cdot|s_t)\right] &\geq \epsilon\mathbb E_{s_t, a\sim\mu(\cdot|s_t)}\left[\sum_{a'\in N^{\text{out}}(a)}\mI_t\left(a_t^*;Y_{t,a'}\right)\right]\\
     &\geq \epsilon\mathbb E_{s_t, a\sim\mu(\cdot|s_t)}\left[\sum_{a'\in N^{\text{out}}(a)}\mI_t\left(a_t^*;Y_{t,a'}\right)\ind\left\{ a'\in N^{\text{out}}(a)\right\}\right]\\
     &\geq \epsilon\mathbb E_{s_t, a\sim\mu(\cdot|s_t)}\left[\sum_{a'\in \cA_t}\mI_t\left(a_t^*;Y_{t,a'}\right)\ind\left\{a'\in N^{\text{out}}(a)\right\}\right]\,.
\end{split}
\end{equation*}
Then we have
\begin{equation}\label{eqn:lower_bound_ig}
     \begin{split}
          \mathbb E_{s_t}\left[\mI_t(\pi^*)^{\top}\pi_t^{\mix}(\cdot|s_t)\right] 
          &\geq \epsilon\mathbb E_{s_t, a\sim \mu(\cdot|s_t)}\left[\sum_{a'\in \cA_t}\mI_t(a_t^*; Y_{t,a'})\right]\min_{a'}\mathbb E_{s_t, a\sim \mu(\cdot|s_t)}\left[\ind\left\{a'\in N^{\text{out}}(a)\right\}\right]\\
          &\geq \epsilon\vartheta(\cG, \xi) \mathbb E_{s_t}\left[\sum_{a'\in \cA_t}\mI_t(a_t^*; Y_{t,a'})\right]\\
          &\geq \frac{2\epsilon\vartheta(\cG, \xi)}{R_{\max}^2+1}\mathbb E_{s_t}\left[\sum_{a\in\cA_t}\mathbb P_t(a_t^*=a)\sum_{a'\in\cA_t}\left(\mathbb E_t[Y_{t, a'}|a_t^*=a]-\mathbb E_t[Y_{t,a'}]\right)^2\right]\,,
     \end{split}
 \end{equation}
 where the last inequality is from Pinsker's inequality. 
 
 On the other hand, by the definition of mixture policy and Jenson's inequality,
 \begin{equation*}
 \begin{split}
     \mathbb E_{s_t}[ 
     \Delta_t(s_t)^{\top}\pi_t^{\mix}(\cdot|s_t)] &\leq (1-\epsilon)\mathbb E_{s_t, a\sim \pi_t^{\TS}(\cdot|s_t)}\mathbb E_t[Y_{t, a_t^*}-Y_{t,a}]+2\epsilon R_{\max} \\
     &\leq  \mathbb E_{s_t, a\sim \pi_t^{\TS}(\cdot|s_t)}\left[\mathbb E_t[Y_{t, a}|a_t^*=a]-\mathbb E_t[Y_{t,a}]\right]+2\epsilon R_{\max}\\
     &\leq\sqrt{\mathbb E_{s_t}\left[\sum_{ a\in\cA_t}\mathbb P_t(a_t^*=a)\left(\mathbb E_t[Y_{t, a}|a_t^*=a]-\mathbb E_t[Y_{t,a}]\right)^2\right]}+2\epsilon R_{\max}\,.
      \end{split}
 \end{equation*}
Combining with the lower bound of the information gain in Eq.~\eqref{eqn:lower_bound_ig}, 
\begin{equation*}
    \mathbb E_{s_t}[ 
     \Delta_t(s_t)^{\top}\pi_t^{\mix}(\cdot|s_t)] \leq  \sqrt{\frac{R_{\max}^2+1}{2\epsilon\vartheta(\cG, \xi)} \mathbb E_{s_t}\left[\mI_t(\pi^*)^{\top}\pi_t^{\mix}(\cdot|s_t)\right] }+2\epsilon R_{\max}\,.
\end{equation*}
By choosing 
\begin{equation*}
    \epsilon = \left(\frac{R_{\max}^2+1}{R_{\max}^2}\frac{1}{8\vartheta(\cG, \xi)}\mathbb E\left[\mI_t(\pi^*)^{\top}\pi_t^{\mix}(\cdot|s_t)\right] \right)^{1/3}\,,
\end{equation*}
we reach
\begin{equation*}
        \mathbb E_{s_t}[ 
     \Delta_t(s_t)^{\top}\pi_t^{\mix}(\cdot|s_t)] \leq 2R_{\max}\left(\frac{R_{\max}^2+1}{R_{\max}^2}\frac{1}{8\vartheta(\cG, \xi)}\mathbb E_{s_t}\left[\mI_t(\pi^*)^{\top}\pi_t^{\mix}(\cdot|s_t)\right] \right)^{1/3}\,.
\end{equation*}
This implies 
\begin{equation*}
    \frac{ \left(\mathbb E_{s_t}[ 
     \Delta_t(s_t)^{\top}\pi_t^{\mix}(\cdot|s_t)]\right)^3 }{\mathbb E_{s_t}\left[\mI_t(\pi^*)^{\top}\pi_t^{\mix}(\cdot|s_t)\right]}\leq \frac{R_{\max}^3+R_{\max}}{\vartheta(\cG, \xi)}\,.
\end{equation*}
This ends the proof.
 \end{proof}
\subsection{Proof of Lemma \ref{lemma:cumulative_ig}}\label{proof:cumulative_ig}
\begin{proof}
According the the definition of cumulative information gain,
\begin{equation}\label{eqn:cumulative_inf_gain}
    \mathbb E\left[\sum_{t=1}^n\mathbb E_{s_t}\left[\mI_t(\pi^*)^{\top}\pi_t(\cdot|s_t)\right]\right] = \mathbb E_{t,s_t}[\mI_t(\pi^*; O_t|A_t, s_t)] \,.
\end{equation}
Note that 
\begin{equation*}
    \mI_t(\pi^*;(s_t, A_t, O_t)) = \mathbb E_{t,s_t}[\mI_t(\pi^*; O_t|A_t, s_t)] + \mathbb E_{t,s_t}[\mI_t(\pi^*; A_t|s_t)] + \mI_t(\pi^*; s_t) =\mathbb E_{t,s_t}[\mI_t(\pi^*; O_t|A_t, s_t)] \,.
\end{equation*}
By the chain rule,
\begin{equation*}
\begin{split}
     \mI(\pi^*; \cF_{n+1}) &= \sum_{t=1}^n\mathbb E\left[\mI_{t}(\pi^*; (s_t, A_t, O_t))\right]= \sum_{t=1}^n\mathbb E\left[\mathbb E_{s_t}\left[\mI_t(\pi^*)^{\top}\pi_t(\cdot|s_t)\right]\right]\,.
\end{split}
\end{equation*}
On the other hand,
\begin{equation*}
     \mI(\pi^*; \cF_{n+1}) \leq \mathbb H(\pi^*)\leq M\log(k)\,,
\end{equation*}
where $\mathbb H(\cdot)$ is the entropy.
Putting the above together,
\begin{equation*}
     \mathbb E\left[\sum_{t=1}^n\mathbb E_{s_t}\left[\mI_t(\pi^*)^{\top}\pi_t(\cdot|s_t)\right]\right] \leq M\log(k).
\end{equation*}
This ends the proof.
\end{proof}
\begin{remark}
We analyze the upper bound $\mH(\pi^*)$ under the following examples. Denote the number of contexts by $M$, which could be infinite.
\begin{enumerate}
    \item The possible choices of $\pi^*$ is at most $k^M$, and thus $\mH(\pi^*)\leq \log(k^M) = M\log(k)$. When $M = \infty$, this bound is vacuous.
    \item 
    We can divide the context space into fixed and disjoint clusters such that for each parameter in the parameter space, the optimal policy maps all the contexts in a single cluster to a single action. For example, a naive partition is that each cluster contains a single context, and in this case, the number of clusters denoted equals the number of contexts $M$. 
    Let $P$ be the minimum number of such clusters, so $P \leq M$.
    The possible choice of $\pi^*$ is $k^{P}$, and thus $\mH(\pi^*)\leq \log(k^{P}) = {P}\log(k)$. When $M=\infty$ but ${P} < \infty$, we obtain a valid upper bound instead of $\infty$.
\end{enumerate}
\end{remark}
\subsection{Proof of Theorem \ref{thm:main_graph}} 
\begin{proof}
    According to Theorem \ref{thm:generic_bound_ids}, we have 
    \begin{equation*}
 \begin{split}
     & \BR(n; \pi^{\text{MIR}})\leq 2R_{\max}\sqrt{n}+\inf_{\lambda\geq 2}2^{1-2/\lambda}\cI_{2R_{\max}/\sqrt{n},\lambda}^{1/\lambda}n^{1-1/\lambda}\mathbb E\left[\sum_{t=1}^n\mathbb E\left[\mI_t(\pi^*)^{\top}\pi_t(\cdot|s_t)\right]\right]^{1/\lambda}\,.
 \end{split}
 \end{equation*}
 Using Lemma \ref{lemma:graph_ir_2} with $\epsilon=1/\sqrt{n}$, we have
 \begin{equation*}
      \cI_{2R_{\max}/\sqrt{n},2}\leq \frac{2R_{\max}^2+2}{1-1/\sqrt{n}}2\beta(\cG)\log\left(\frac{4k^2\sqrt{n}}{\beta(\cG)}\right)\leq 16(R_{\max}^2+1)\beta(\cG)\log\left(\frac{4k^2\sqrt{n}}{\beta(\cG)}\right)\,.
 \end{equation*}
 Combining Lemmas \ref{lemma:graph_ir_3}-\ref{lemma:cumulative_ig} together, we have 
 \begin{equation*}
     \begin{split}
        &\BR(n; \pi^{\text{MIR}})\leq  2R_{\max}\sqrt{n} +\\
        &\min\left(\left(16(R_{\max}^2+1)\beta(\cG)\log\left(\frac{4k^2\sqrt{n}}{\beta(\cG)}\right)nM\log(k)\right)^{1/2}, \left(\frac{R_{\max}^3+R_{\max}}{\vartheta(\cG, \xi)}2M\log(k)\right)^{\tfrac{1}{3}}n^{\tfrac{2}{3}}\right)\\
        &\geq CR_{\max}\min\left( \sqrt{\beta(\cG)\log\left(\frac{4k^2\sqrt{n}}{\beta(\cG)}\right)nM\log(k)}, \left(\frac{2M\log(k)}{\vartheta(\cG, \xi)}\right)^{\tfrac{1}{3}}n^{\tfrac{2}{3}}\right)\,,
     \end{split}
 \end{equation*}
 for some absolute constant $C>0$. This ends the proof.
\end{proof} 
 
\input{proof_sparse}

%% file: proof_sparse.tex
\section{Sparse Linear Contextual Bandits}
\subsection{Proof of Theorem \ref{thm:lower_bound_sparse}}
\begin{proof}
    We assume there are two available contexts. Context set 1 consists of a single action $x_0=(0, \ldots, 0)^{\top}\in\mathbb R^{d+1}$ and an informative action set $\cH$ as follows:
\begin{equation}\label{eqn:action_set}
    \begin{split}
       \cH = \Big\{x\in\mathbb R^{d+1}\big|x_j\in\{-\kappa, \kappa\} \ \text{for} \ j\in[d], x_d = 1\Big\} \,,
    \end{split}
\end{equation}
where $0<\kappa\leq1$ is a constant. Let $d=sp$ for some integer $p\geq 2$. Context set 2 consists of a multi-task bandit action set $\cA = \{(\{e_i\in\mathbb R^p: i\in[p]\}^s, 0)\}\subset \mathbb R^{d+1}$ whose element's last coordinate is always 0. The arriving probability of each context is 1/2. One can verify for this feature set, the explorability constant $C_{\min}(\phi, \xi) = 1/2$ achieved by a policy that uniformly samples from $\cH$ when context set 1 arrives.

Fix a conditional IDS policy $\pi^{\text{CIR}}$. Let $\Delta>0$ and $\Theta=\{\Delta e_i:i\in[p]\}\subset \mathbb R^p$. Given $\theta\in\{(\Theta^s, 1)\}\subset \mathbb R^{d+1}$ and $i\in[s]$, let $\theta^{(i)}\in\mathbb R^p$ be defined by $\theta_k^{(i)}=\theta_{(i-1)s+k}$, which means that 
\begin{equation*}
    \theta^{\top} = [\theta^{(1)\top}, \ldots, \theta^{(s)\top}, 1]\,.
\end{equation*}
Assume the prior of $\theta^*$ is uniformly distributed over $\{(\Theta^s, 1)\}$. Define the cumulative regret of policy $\pi$ interacting with bandit $\theta$ as
\begin{equation}\label{eqn:sparse_lower_bound1}
\begin{split}
    R_{\theta}(n; \pi) &= \sum_{t=1}^n\mathbb E_{\theta}\left[\langle x_{s_t}^*, \theta\rangle-Y_t\right]\\
    &=\sum_{t=1}^n\mathbb E_{\theta}\left[\left(\langle x_{s_t}^*, \theta\rangle-Y_t\right)\ind(s_t=1)\right]+\sum_{t=1}^n\mathbb E_{\theta}\left[\left(\langle x_{s_t}^*, \theta\rangle-Y_t\right)\ind(s_t=2)\right]\,,
    \end{split}
\end{equation}
where we write $x_{s_t}^*=\phi(s_t, a_t^*)$ for short. Therefore,
\begin{equation*}
\begin{split}
     \BR(n; \pi)= \frac{1}{|\Theta|^s}\sum_{\theta\in\{(\Theta^s, 1)\}}R_{\theta}(n;\pi) \,.
\end{split}
\end{equation*}

Note that when context set 1 arrives, the action from $\cH$ suffers at least $1-s\Delta$ regret and thus since $x_{0}$ is always the optimal action for context set 1. 
From the definition of conditional IDS in Eq.~\eqref{def:conditional_IDS}, when context set 1 is arriving and $s\Delta<1$, conditional IDS will always pull $x_{0}$ for this prior. That means conditional IDS will suffer no regret for context set 1 and implies for any $\theta\in \{(\Theta^s, 1)\}$,
\begin{equation*}
\begin{split}
     &\sum_{t=1}^n\mathbb E_{\theta}\left[\left(\langle x_{s_t}^*, \theta\rangle-Y_t\right)\mathbb I(s_t=1)\right]
     =\sum_{t=1}^n\mathbb E_{\theta}\left[\langle x_{1}^*, \theta\rangle-\langle \pi^{\text{CIR}}(i|1), \theta\rangle|s_t=1\right]\mathbb P(s_t=1)=0\,.
\end{split}
\end{equation*}

It remains to bound the second term in Eq.~\eqref{eqn:sparse_lower_bound1}. It essentially follows the proof of Theorem 24.3 in \citet{lattimore2020bandit}. %For the  self-completeness, we inlcude a full proof here. Let $Z_t\in[p]^s$ as the vector of base actions chosen by the learner in each of the $s$ bandits in round $t$. The optimal action in the $i$th bandits is $b_i^*(\theta)=\argmax_{b\in[k]\theta_b^{(i)}}$.
From the proof of multi-task bandit lower bound, we have 
\begin{equation*}
    \frac{1}{|\Theta|^s}\sum_{\theta\in\{(\Theta^s, 1)\}} \sum_{t=1}^n\mathbb E_{\theta}\left[\left(\langle x_{s_t}^*, \theta\rangle-Y_t\right)\ind(s_t=2)\right] \geq \frac{1}{16}\sqrt{dsn}\,.
\end{equation*}
This ends the proof.
\end{proof}

\subsection{Proof of Lemma \ref{lemma:information_ratio_sparse}}
\begin{proof}
If one can derive a worst-case bound of $ \Psi_{t,\alpha}^{\lambda}(\tilde{\pi})$ for a particular policy $\tilde{\pi}$,, we can have an upper bound for $\cI_{\alpha,\lambda}$ automatically. The remaining step is to choose a proper policy $\tilde{\pi}$ for $\lambda=2,3$ separately.

First, we bound the information ratio with $\lambda=2$. By the definition of mutual information, for any $a\in\cA_t$, we have 
\begin{equation}\label{eqn:IG_pinsker}
    \begin{split}
      \mI_t(a_t^*; Y_{t,a})
      =\sum_{a'\in\cA_t}\mathbb P_t(a_t^*=a')D_{\KL}\left(\mathbb P_t(Y_{t,a}=\cdot|a_t^*=a')||\mathbb P_t(Y_{t,a} = \cdot)\right).
    \end{split}
\end{equation}
Recall that $R_{\max}$ is the upper bound of maximum expected reward. It is easy to see $Y_{t,a}$ is a $\sqrt{R_{\max}^2+1}$ sub-Gaussian random variable. According to Lemma 3 in \citet{russo2014learning}, we have 
\begin{equation}\label{eqn:bound_information_gain}
    \mI_t(a_t^*;Y_{t,a})\geq \frac{2}{R_{\max}^2+1}\sum_{a'\in\cA_t}\mathbb P_t(a_t^*=a')\Big(\mathbb E_t[Y_{t,a}|a_t^*=a']-\mathbb E_t[Y_{t,a}]\Big)^2.
\end{equation}
The MIR of contextual IDS can be bounded by the MIR of Thompson sampling:
\begin{equation*}
    \cI_{0,2}\leq \max_{t\in[n]}\frac{ \max\left(0, \mathbb E[\Delta_t(s_t)^{\top}\pi_t^{\TS}(\cdot|s_t)]\right)^{2}}{\mathbb E[\mI_t(\pi^*)^{\top}\pi_t^{\TS}(\cdot|s_t)]}\leq \max_{t\in[n]}\frac{ \max\left(0, \mathbb E[\Delta_t(s_t)^{\top}\pi_t^{\TS}(\cdot|s_t)]\right)^{2}}{\mathbb E[\mI_t(a_t^*)^{\top}\pi_t^{\TS}(\cdot|s_t)]}\leq \max_{t\in[n]}\mathbb E\left[\frac{ (\Delta_t(s_t)^{\top}\pi_t^{\TS}(\cdot|s_t))^{2}}{\mI_t(a_t^*)^{\top}\pi_t^{\TS}(\cdot|s_t)}\right] ,
\end{equation*}
where the first inequality is from data processing inequality and the second inequality is from Jenson's inequality. Using the matrix trace rank trick described in Proposition 5 in \citet{russo2014learning}, we have $ \cI_{0,2}\leq (R_{\max}^2+1)d/2$ in the end.

Second, we bound the information ratio with $\lambda=3$. Let's define an exploratory policy $\mu$ such that
\begin{equation}\label{def:explortory}
    \mu =  \argmax_{\pi:\cS\to \cP(\cA)} \sigma_{\min}\left( \mathbb E_{s\sim \xi}\Big[\mathbb E_{a\sim \pi(\cdot|s)}[\phi(s, a)\phi(s, a)^{\top}]\Big]\right)\,.
\end{equation}
Consider a mixture policy $\pi_t^{\mix} = (1-\epsilon)\pi_t^{\text{TS}}+\epsilon \mu$ where the mixture rate $\epsilon\geq 0$ will be decided later. 

\paragraph{Step 1: Bound the information gain}
According the lower bound of information gain in Eq.~\eqref{eqn:bound_information_gain},
\begin{equation*}
\begin{split}
     &\mathbb E_{s_t}\left[\mI_t(a_t^*)^{\top}\pi_t^{\mix}(\cdot|s_t)\right]\\
     \geq& \frac{2}{R_{\max}^2+1}\mathbb E_{s_t\sim\xi, a\sim \pi_t^{\mix}(\cdot|s_t)}\left[\sum_{a'\in\cA_t}\mathbb P_t(a_t^*=a')\left(\mathbb E_t[Y_{t,a}|a_t^*=a']-\mathbb E_t[Y_{t,a}]\right)^2\right]\\
     =& \frac{2}{R_{\max}^2+1}\mathbb E_{s_t\sim\xi, a\sim \pi_t^{\mix}(\cdot|s_t)}\left[\sum_{a'\in\cA_t}\mathbb P_t(a_t^*=a')\left(\phi(s_t, a)^{\top}\mathbb E_t[\theta^*|a_t^*=a']-\phi(s_t, a)^{\top}\mathbb E_t[\theta^*]\right)^2\right]\,.
\end{split}
\end{equation*}
By the definition of the mixture policy, we know that $\pi_t^{\mix}(a|s_t)\geq \epsilon \mu(a|s_t)$ for any $a\in\cA_t$. Then we have 
\begin{equation*}
\begin{split}
    &\mathbb E_{s_t}\left[\mI_t(a_t^*)^{\top}\pi_t^{\mix}(\cdot|s_t)\right]\geq \frac{2\epsilon}{R_{\max}^2+1}\sum_{a'\in\cA_t}\mathbb P_t(a_t^*=a')\\
     &\cdot\mathbb E_{s_t\sim\xi, a\sim \mu(\cdot|s_t)}\Big[(\mathbb E_t[\theta^*|a_t^*=a']-\mathbb E_t[\theta^*])^{\top}\phi(s_t, a)\phi(s_t, a)^{\top}(\mathbb E_t[\theta^*|a_t^*=a']-\mathbb E_t[\theta^*])\Big]\,.
\end{split}
\end{equation*}
From the definition of $\mu$ in Eq.~\eqref{def:explortory}, we have
\begin{equation*}
    \mathbb E_{s_t}\left[\mI_t(a_t^*)^{\top}\pi_t^{\mix}(\cdot|s_t)\right]\geq \frac{2\epsilon}{R_{\max}^2+1}\sum_{a'\in\cA_t}\mathbb P_t(a_t^*=a')C_{\min}(\phi, \xi)\big\|\mathbb E_t[\theta^*|a_t^*=a']-\mathbb E_t[\theta^*]\big\|_2^2\,.
\end{equation*}

\paragraph{Step 2: Bound the instant regret} We decompose the regret by the contribution from the exploratory policy and the one from TS:
\begin{equation}\label{eqn:bound1}
\begin{split}
      \mathbb E_{s_t}\left[\Delta_t(s_t)^{\top}\pi_t^{\mix}(\cdot|s_t)\right]
      &= (1-\epsilon)\mathbb E_{s_t}\left[\sum_a \mathbb P_t(a_t^*=a)\Big(\mathbb E_t\left[\phi(s_t, a)^{\top} \theta^*\big|a_t^*=a\right]- \mathbb E_t\left[\phi(s_t, a)^{\top}\theta^* \right]\Big)\right]\\
      &+ \epsilon \mathbb E_{s_t}\left[\sum_a\mathbb E_t\Big[\phi(s_t,a_t^*)^{\top}\theta^* -\phi(s_t,a)^{\top}\theta^*\Big]\mu(a|s_t)\right]\,.
\end{split}
\end{equation}
Since $R_{\max}$ is the upper bound of maximum expected reward, the second term can be bounded $2R_{\max}\epsilon$. Using Jenson's inequality, the first term can be bounded by
\begin{equation*}
    \begin{split}
        &\mathbb E_{s_t}\left[\sum_a \mathbb P_t(a_t^*=a)\Big(\mathbb E_t\left[\phi(s_t, a)^{\top} \theta^*\big|a_t^*=a\right]- \mathbb E_t\left[\phi(s_t, a)^{\top}\theta^* \right]\Big)\right] \\
        &\leq \sqrt{\mathbb E_{s_t}\left[\sum_a \mathbb P_t(a_t^*=a)\Big(\mathbb E_t\left[\phi(s_t, a)^{\top} \theta^*\big|a_t^*=a\right]- \mathbb E_t\left[\phi(s_t, a)^{\top}\theta^* \right]\Big)^2\right]}\,.
    \end{split}
\end{equation*}
Since all the optimal actions are sparse, any action $a$ with $\mathbb P_t(a_t^*=a)>0$ must be sparse. Then we have 
\begin{equation*}
    \Big(\phi(s_t,a)^{\top}(\mathbb E_t[\theta^*|a_t^*=a]-\mathbb E_t[\theta^*])\Big)^2\leq s^2 \Big\|\mathbb E_t[\theta^*|a_t^*=a]-\mathbb E_t[\theta^*]\Big\|_2^2\,,
\end{equation*}
for any action $a$ with $\mathbb P_t(a_t^*=a)>0$. This further implies
\begin{equation}\label{eqn:bound2}
\begin{split}
     &\mathbb E_{s_t}\left[\sum_a \mathbb P_t(a_t^*=a)\Big(\mathbb E_t\left[\phi(s_t, a)^{\top} \theta^*\big|a_t^*=a\right]- \mathbb E_t\left[\phi(s_t, a)^{\top}\theta^* \right]\Big)\right]\\
     &\leq \sqrt{\mathbb E_{s_t}\left[\sum_a \mathbb P_t(a_t^*=a)s^2\Big\|\mathbb E_t[\theta^*|a_t^*=a]-\mathbb E_t[\theta^*]\Big\|_2^2\right]}\\
     &=\sqrt{\frac{s^2(R_{\max}^2+1)}{2\epsilon C_{\min}(\phi,\xi)}\frac{2\epsilon C_{\min}(\phi,\xi)}{(R_{\max}^2+1)}\mathbb E_{s_t}\left[\sum_{a\in\cA_t}\mathbb P_t(a_t^*=a)s^2\Big\|\mathbb E_t[\theta^*|a_t^*=a]-\mathbb E_t[\theta^*]\Big\|_2^2\right]}\\
     &\leq \sqrt{\frac{s^2(R_{\max}^2+1)}{2\epsilon C_{\min}(\phi,\xi)} \mathbb E\left[\mI_t(a_t^*)^{\top}\pi_t^{\mix}(\cdot|s_t)\right]}\,.
    \end{split}
\end{equation}
Putting Eq.~\eqref{eqn:bound1} and \eqref{eqn:bound2} together, we have
\begin{equation*}
   \mathbb E\left[\Delta_t(s_t)^{\top}\pi_t^{\mix}(\cdot|s_t)\right]\leq \sqrt{\frac{s^2(R_{\max}^2+1)}{2\epsilon C_{\min}(\phi,\xi)} \mathbb E\left[\mI_t(a_t^*)^{\top}\pi_t^{\mix}(\cdot|s_t)\right]} + 2R_{\max}\epsilon\,.
\end{equation*}
By optimizing the mixture rate $\epsilon$, we have 
\begin{equation*}
   \cI_{0,3}\leq \frac{\left(\mathbb E_{s_t}\left[\Delta_t(s_t)^{\top}\pi_t^{\mix}(\cdot|s_t)\right]\right)^3}{\mathbb E_{s_t}\left[\mI_t(\pi^*)^{\top}\pi_t^{\mix}(\cdot|s_t)\right]}\leq  \frac{\left(\mathbb E_{s_t}\left[\Delta_t(s_t)^{\top}\pi_t^{\mix}(\cdot|s_t)\right]\right)^3}{\mathbb E_{s_t}\left[\mI_t(a_t^*)^{\top}\pi_t^{\mix}(\cdot|s_t)\right]}\leq \frac{s^2(R_{\max}^2+1)}{8R^2_{\max}C_{\min}}\leq \frac{s^2}{4C_{\min}(\phi,\xi)}\,.
\end{equation*}
This ends the proof.
\end{proof}